\documentclass[journal,draftclsnofoot,12pt,onecolumn]{IEEEtran}
\usepackage{amsmath,amssymb,amsthm,mathtools,bm}
\usepackage{dsfont}
\usepackage{graphicx}
\usepackage{booktabs}
\usepackage{multirow,caption}
\usepackage{algorithm}
\usepackage{algpseudocode}
\usepackage{cite}
\usepackage{microtype}
\usepackage[colorlinks=true,
            linkcolor=blue,
            citecolor=blue,
            urlcolor=blue]{hyperref}
\usepackage[nameinlink,capitalize]{cleveref}

\newtheorem{proposition}{Proposition}

\title{Confounding-Valid Conformal Inference for Counterfactual KPIs in Wireless Networks}
\date{}
\author{Abdessamed Qchohi, Jessica Moysen Cortes, and Matteo Zecchin
\thanks{
Abdessamed Qchohi and Matteo Zecchin are with the Communication Systems Department, EURECOM, 
06904 Sophia Antipolis, France  (e-mail: \{abdessamed.qchohi, matteo.zecchin\}@eurecom.fr). Jessica Moysen Cortes is with Huawei Technologies Sweden AB, Sweden (e-mail: jessica.moysen.cortes@huawei.com).
}
}

\begin{document}

\maketitle
\begin{abstract}
Conformal counterfactual inference enables network operators to use logged telemetry to reliably answer “what-if” questions about network operation. These answers typically take the form of prediction sets that contain, with a user-defined probability, the key performance indicators (KPIs) that would have been observed under alternative control actions. A key challenge is that logged telemetry may omit variables used by the controller, resulting in hidden confounding and invalidating the statistical guarantees of counterfactual analysis. In principle, this issue can be addressed using randomized telemetry, collected by assigning control actions independently of the network state. However, because such randomization may disrupt normal operation, randomized telemetry is typically scarce, causing counterfactual analysis based solely on it to produce uninformative prediction sets. To address these challenges, we propose \emph{Confounding-Valid Counterfactual Conformal Inference} (CV-CCI), which combines abundant, potentially confounded observational telemetry with limited randomized data through the General Synthetic-Powered Inference (GESPI) principle. CV-CCI leverages observational data to improve efficiency while using randomized data to retain finite-sample coverage guarantees under arbitrary hidden confounding. Experiments on two representative radio access network (RAN) control tasks show that CV-CCI remains valid under hidden confounding while producing more efficient prediction sets than state-of-the-art confounding-valid baselines.
\end{abstract}


\section{Introduction}

\subsection{Motivation}

Modern radio access networks (RANs), including those based on Open RAN (O-RAN) architectures, increasingly rely on intelligent controllers to adapt network operation to changing traffic demands, channel conditions, and service requirements \cite{bonati2021intelligence,polese2023colo}. Because these control decisions directly affect network key performance indicators (KPIs), operators may ask counterfactual questions such as: \emph{What KPI would have been observed had a different control action been taken?} The practical value of answering such ``what-if'' questions depends critically on reliable uncertainty quantification. Operators need not only a prediction of the KPI under an alternative action, but also a faithful measure of its uncertainty, typically expressed through an interval or prediction set designed to contain the target outcome with a prescribed probability \cite{hou2025if,simeone2026conformal}.

A key difficulty, however, is that the telemetry available for counterfactual analysis may not contain all the information used by the controller to select its action. For example, the controller may rely on fine-grained measurements, lower-layer information, private internal states, or signals exchanged with other controllers, whereas the logged telemetry may contain only aggregated, delayed, or filtered observations \cite{LACAVA2025111342,abdalla2022toward,Haohuang}. If an omitted variable affects both the selected action and the resulting KPI, the logged data are subject to hidden confounding \cite{hernan2010causal}. In that case, samples associated with a given action are not representative of what would have occurred had that action been assigned under the observed context, and prediction sets calibrated solely from observational telemetry may fail to satisfy their intended coverage requirements \cite{chen2024}.

One possible solution to this problem consists of applying counterfactual analysis to randomized telemetry, namely, data obtained by assigning control actions independently of both the observed network context and any hidden information available to the controller. This randomization removes the dependence between the action and hidden confounders, thereby allowing valid counterfactual coverage to be recovered. In a live network, however, such interventions may be costly, suboptimal, or incompatible with operational constraints. As a result, randomized telemetry is typically scarce, and prediction sets calibrated exclusively on it may be uninformative and highly variable. 

\begin{figure*}[t]
    \centering
    \includegraphics[width=0.85\linewidth]{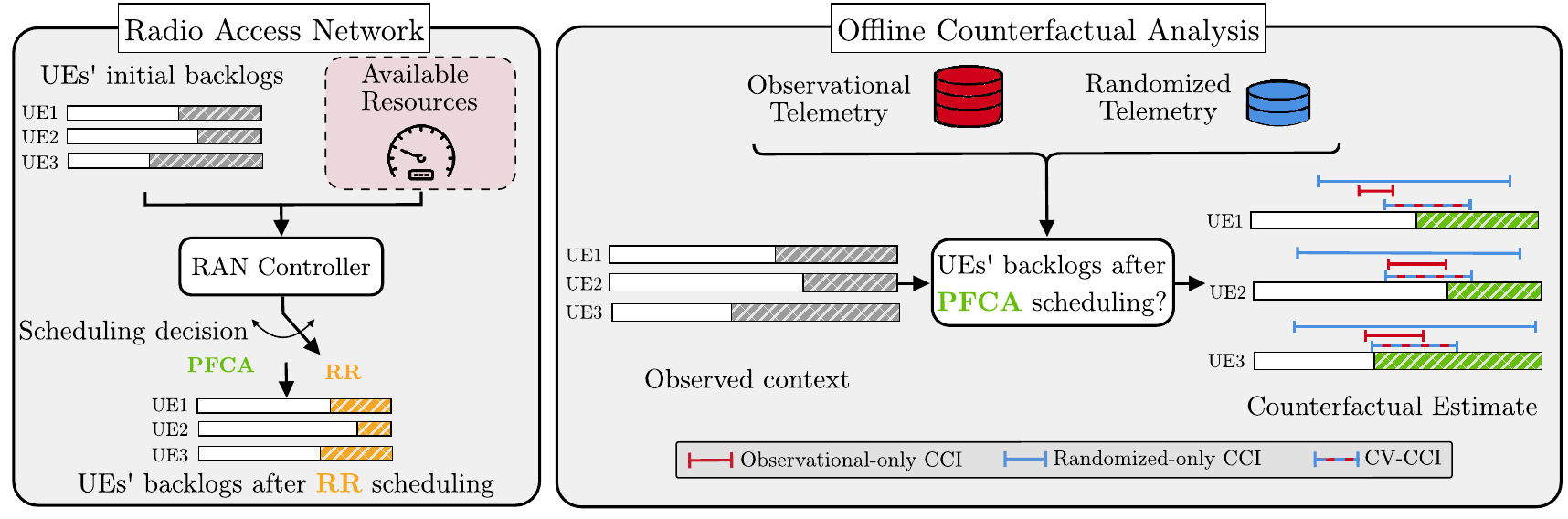}
    \caption{A RAN controller selects between proportional-fair channel-aware (PFCA) and round-robin (RR) scheduling based on the UEs' backlogs and available radio resources. The network operation is logged via telemetry that records the backlogs, selected scheduler, and residual backlog after scheduling, but omits resource availability. Counterfactual analysis can use abundant observational telemetry collected during normal operation and limited randomized telemetry in which the scheduler is selected independently of the network context. The goal is to construct a prediction set for the residual backlog that would have been observed under the alternative scheduler. Prediction intervals calibrated exclusively on observational telemetry (red) may fail to cover the ground-truth residual backlog (green), whereas those calibrated exclusively on randomized telemetry (blue) are valid but wide. CV-CCI combines both data sources to produce valid and more informative prediction sets (red-and-blue dashed).}
    \label{fig:motivating_example}
\end{figure*}

These issues are illustrated through a scheduling example in Figure \ref{fig:motivating_example}. At each scheduling frame, a network controller selects either round-robin (RR) or proportional-fair channel-aware (PFCA) scheduling based on the users' backlog information and the available radio resources. The logged telemetry, however, records only the user-level context and omits the instantaneous radio-resource budget. The goal of reliable counterfactual analysis is to construct a prediction set that contains, with a prescribed probability, the residual backlog that would have been observed under the alternative scheduling policy. Because the unobserved resource budget affects both the policy selected by the controller and the resulting residual backlog, prediction sets constructed solely from observational data may fail to provide valid coverage for the counterfactual residual backlog, as shown in the right panel of Figure \ref{fig:motivating_example}. Valid coverage can instead be recovered using randomized telemetry, collected by assigning the scheduling policy independently of both UEs' backlog information and resource availability. Such randomization may degrade network performance, however, and can therefore be applied only sparingly. As illustrated in the right panel of Figure \ref{fig:motivating_example}, the limited randomized data yield valid but wide prediction sets that provide little information about the counterfactual residual backlog.

The above discussion highlights a fundamental tension between validity and efficiency in counterfactual analysis under hidden confounding. Observational telemetry is abundant and informative but may fail to satisfy the desired coverage guarantee, whereas randomized telemetry supports valid inference but is scarce and costly to collect. Motivated by this tension, we propose \emph{Confounding-Valid Counterfactual Conformal Inference} (CV-CCI), a conformal methodology that combines abundant, potentially confounded observational telemetry with limited randomized telemetry. Specifically, CV-CCI exploits observational data to improve the efficiency of the prediction sets while using randomized data as a validity guardrail against arbitrary hidden confounding. As illustrated in Figure \ref{fig:motivating_example}, this combination retains finite-sample coverage while producing more informative  and stable prediction sets than methods based solely on randomized telemetry.

\subsection{Related Work}

\paragraph{Counterfactual Conformal Inference}
Counterfactual inference addresses the longstanding causal problem of estimating the potential outcomes that would have been observed under actions other than the one actually taken \cite{rubin1974estimating}. Since only the outcome under the realized action is observed, reliable counterfactual analysis requires quantifying uncertainty about the missing potential outcomes. Conformal prediction (CP) uses a calibration sample to convert any predictive model into a set predictor with marginal coverage guarantees, under the assumption that the calibration and test data are exchangeable \cite{vovk2005algorithmic,shafer2008tutorial}. In counterfactual settings, however, action selection generally induces a distribution shift between the observed sample and the target interventional data, violating the exchangeability assumption. Weighted CP (WCP) accounts for this shift by weighting calibration samples based on propensity score estimates \cite{tibshirani2019}, enabling uncertainty quantification for counterfactual outcomes \cite{lei2021}. These guarantees rely on causal identification assumptions, including the absence of hidden confounding. Conformal sensitivity analyses address departures from this assumption by bounding the severity of hidden confounding and translating these bounds into adjusted uncertainty sets \cite{yin2022conformal,jin2023sensitivity}. Closely related work combines observational and randomized data through WCP using an estimated density ratio between the two distributions \cite{chen2024}. In case of high-dimensional telemetry, however, ratio-estimation errors may compromise coverage and reduce efficiency. 

\paragraph{Counterfactual Inference in Wireless Networks}
Causal reasoning has been proposed as a foundation for generalizable and explainable AI-native wireless networks \cite{thomas2024causal}. A central capability is counterfactual, or ``what-if,'' analysis, which  asks how network key performance indicators (KPIs) would change under an alternative configuration, control policy, or RAN application. Digital twins provide a natural platform for such analyses by maintaining a virtual representation of the physical network that is synchronized with network telemetry \cite{khan2022digital,lin20236g,ruah2023bayesian}. Operators can query the twin to compare candidate interventions before deployment, without disrupting the live network \cite{ak2024if}. Closer to our setting, conformal counterfactual KPI estimation (CCKE) constructs prediction sets directly from live-network telemetry for the KPI that would have been observed under an alternative RAN application or action \cite{hou2025if}. CCKE provides prediction sets with finite-sample coverage under the assumption that the logged context captures all variables affecting action selection. Our work addresses the more general setting in which the controller also relies on unlogged network state, thereby introducing hidden confounding.

\paragraph{Inference with Auxiliary Data}
A growing literature develops inference procedures that combine a small dataset drawn from the target distribution with a larger dataset obtained from an auxiliary source.  Prediction-powered inference follows this principle by combining labeled data with model predictions on abundant unlabeled data and using the labeled data to correct prediction errors \cite{angelopoulos2023prediction}. A prediction-powered correction has also been used to improve the efficiency of conformal counterfactual inference under treatment imbalance, where treated calibration data are scarce \cite{farzaneh2025synthetic}. Related approaches extend this principle to treatment-effect estimation across experiments \cite{cadei2026prediction}, while other work combines randomized and observational studies for treatment-effect estimation \cite{rosenman2023combining,wu2022integrative}. More recently, General Synthetic-Powered Inference (GESPI) has extended this paradigm to a broad class of statistical procedures, using the small target-distribution sample as a guardrail against arbitrary distributional mismatch in the auxiliary data \cite{bashari2025gespi}. CV-CCI specializes the GESPI principle to wireless counterfactual prediction, with randomized telemetry providing the target-distribution sample and potentially confounded observational telemetry serving as auxiliary data.
\subsection{Contributions}

We introduce \emph{Confounding-Valid Counterfactual Conformal Inference} (CV-CCI) for constructing statistically valid and informative prediction sets for counterfactual network KPIs under hidden confounding. The main contributions are summarized as follows:

\begin{itemize}
    \item We formulate reliable counterfactual KPI inference in wireless networks when the controller selects actions using information that is only partially available in the logged telemetry. This setting captures the hidden confounding that arises from the mismatch between the controller's information and the operator's observations.

    \item We develop CV-CCI, a conformal counterfactual KPI estimation methodology that combines abundant, potentially confounded observational telemetry with limited randomized telemetry through a GESPI-based aggregation rule. We establish finite-sample marginal coverage under arbitrary hidden confounding, up to a user-specified tolerance, without requiring the hidden confounder to be observed or modeled.

    \item We evaluate CV-CCI in representative RAN-control settings involving MAC-layer scheduling and handover decisions. The results show that CV-CCI maintains reliable coverage as the strength of hidden confounding increases, while producing more informative and less variable prediction sets than methods calibrated exclusively on randomized telemetry.
\end{itemize}

\subsection{Organization}

The remainder of the paper is organized as follows. Section~II introduces the system model and formalizes counterfactual KPI inference under hidden confounding. Section~III reviews conformal prediction, weighted conformal prediction, counterfactual conformal inference, and GESPI. Section~IV presents CV-CCI and establishes its theoretical guarantees. Section~V evaluates the proposed methodology in MAC-layer scheduling and handover scenarios. Finally, Section~VI concludes the paper.

\section{Setting and Problem Definition}
\label{sec:problem-definition}

In this section, we formulate the problem of reliable counterfactual KPI estimation under hidden confounding in the context of RAN control.

\begin{figure*}[ht]
    \centering
    \includegraphics[width=0.8\textwidth]{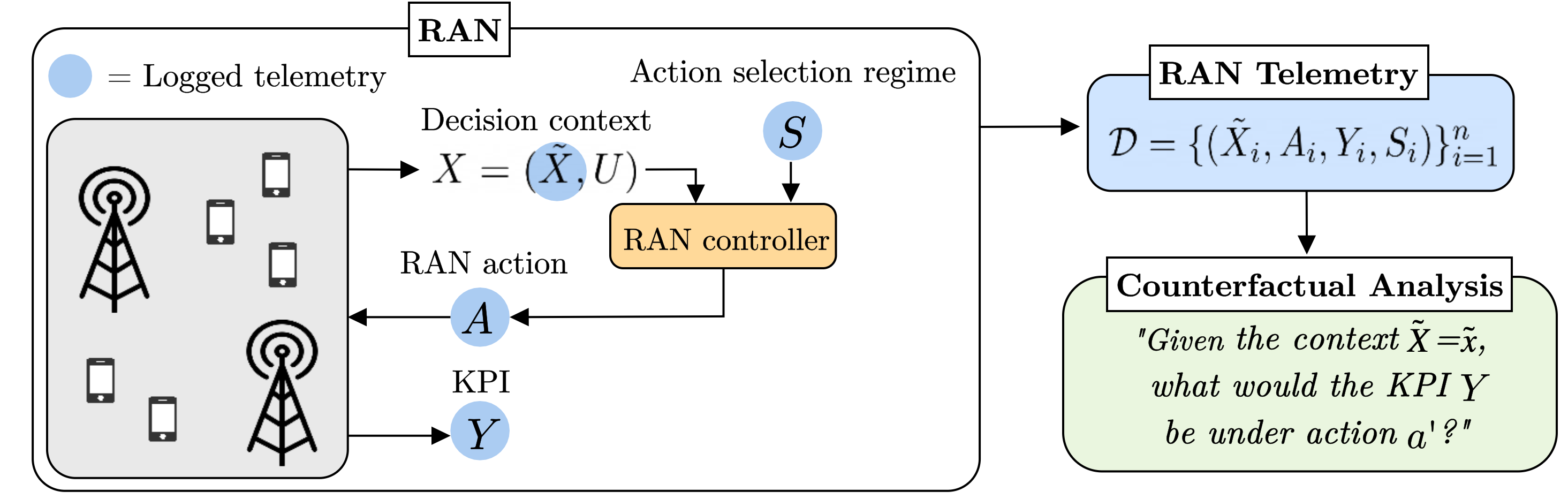}
    \caption{Illustration of the considered counterfactual KPI estimation setting. The controller observes the full network context $X=(\tilde{X},U)$, selects an action $A$ under either the observational or randomized regime $S$, and records the resulting KPI $Y$. Telemetry contains $(\tilde{X},A,Y,S)$ but not the hidden context $U$. Using the recorded telemetry, the goal is to construct a prediction set for the potential KPI $Y(a)$ associated with a target action $a$ and a logged state $\tilde{X}=\tilde{x}$.}
    \label{fig:setting}
\end{figure*}

\subsection{Setting}
\label{subsec:setting}

We consider the closed-loop wireless RAN control system illustrated in Figure \ref{fig:setting}, which operates as follows. At each decision instant, a network controller observes the network context $X$ and selects a control action $A$ from a finite action space $\mathcal{A}$ based on $X$. After action $A$ is applied, the controller records the resulting KPI $Y$. The action $A$ may correspond to any network decision, such as a scheduling decision, a handover command, the selection of an application to deploy in an O-RAN architecture \cite{polese2023understanding}, or the choice of a set of network-configuration parameters. The KPI $Y$ may quantify any performance metric of interest, such as latency, throughput, reliability, or energy efficiency.

Following the potential-outcome framework \cite{rubin1974estimating}, we define, for each action $a\in\mathcal{A}$, the potential KPI that would be observed if action $a$ were selected as $Y(a)\in\mathcal{Y}$. Additionally, under the stable unit treatment value assumption (SUTVA) \cite{rubin1990formal}, the observed KPI corresponding to the selected action $A$ is equal to the associated potential outcome, i.e.,
\begin{align}
    \label{eq:sutva}
    Y = Y(A).
\end{align}
Condition \eqref{eq:sutva} implies that the reported KPI is consistent with the network decision that was actually executed. It also rules out possible perturbations in performance reporting that would cause the observed KPI to differ from the potential KPI associated with the selected action \cite{hou2025if}.

\subsubsection{Action-assignment regimes} 
We consider two mechanisms for assigning the action $A$ and introduce the binary variable $S\in\{0,1\}$ to distinguish between the corresponding action-assignment regimes.

The first regime, corresponding to $S=0$, represents normal network operation. In this regime, the deployed controller selects action $A$ according to an arbitrary, possibly optimized and stochastic policy that may depend on the full decision context $X$. We refer to this assignment regime as \emph{observational}. Under the observational regime,
\begin{align}
    \label{eq:obs_policy}
    A \mid X, S=0
    \sim
    p_{\mathrm{obs}}(\cdot\mid X),
\end{align}
where $p_{\mathrm{obs}}(a\mid X)$ is the probability of selecting action $a\in\mathcal{A}$ given the decision context $X$.

The second regime, corresponding to $S=1$, represents randomized assignment, as in a randomized controlled trial (RCT) \cite{hernan2010causal}. Under the \emph{randomized} regime, the action is assigned independently of the decision context $X$ according to a known policy  $p_{\mathrm{rnd}}(\cdot)$,
\begin{align}
    \label{eq:int_policy}
    A|S=1 \sim p_{\mathrm{rnd}}(\cdot),
    \qquad
    A \perp
    X
    \mid S=1.
\end{align}
Under the randomized policy $p_{\mathrm{rnd}}(\cdot)$ every action has nonzero probability of being selected,
\begin{align}
\label{eq:int_positivity}
p_{\mathrm{rnd}}(a)>0,
\qquad
\text{for all } a\in\mathcal{A},
\end{align}
an assumption known as the positivity condition \cite{hernan2010causal, lei2021}.

We assume that the contexts and potential KPIs have the same distribution under the two regimes and that the regimes differ only in their action-assignment mechanisms. This assumption is satisfied when the choice of action-assignment regime is independent of the context and potential KPIs. For each regime $s\in\{0,1\}$, the conditional joint distribution factorizes as
\begin{align}
\label{eq:potential_sample}
P_{X,A,\{Y(a)\}_{a\in\mathcal A}\mid S=s}
=
P_X
P_{A\mid X,S=s}
P_{\{Y(a)\}_{a\in\mathcal A}\mid X},
\end{align}
where $P_X$ is the distribution of the decision context $X$,
$P_{\{Y(a)\}_{a\in\mathcal A}\mid X}$ is the conditional distribution
of the potential outcomes given $X$, and $P_{A\mid X,S=s}$ is the
action-assignment mechanism under regime $s$, given by
\begin{align}
P_{A\mid X,S=s}(a\mid x)
=
\begin{cases}
p_{\mathrm{obs}}(a\mid x), & s=0,\\
p_{\mathrm{rnd}}(a),       & s=1.
\end{cases}
\end{align}

\subsubsection{Telemetry logging}  The operation of the RAN is recorded through telemetry, with each decision instance logged as the tuple
\begin{align}
\label{eq:telemetry_sample_setting}
(\tilde{X},A,Y,S),
\end{align}
where $\tilde{X}$ denotes the \emph{logged} network state, $A$ is the action selected by the controller, $Y$ is the resulting KPI, and $S$ is the action-assignment regime.
Importantly, the logged state $\tilde{X}$ need not contain all the contextual information used by the controller. To make this mismatch explicit, we decompose the controller's decision context as
\begin{align}
    \label{eq:context_decomposition}
    X = (\tilde{X},U),
\end{align}
where $U$ collects the components of the controller’s information that are not recorded in the logged telemetry. Accordingly, the context distribution $P_X$ corresponds to the joint distribution of the logged and unobserved context components
\begin{align}
    \label{eq:context_distribution}
    P_X
    =
    P_{\tilde{X},U}
    =
    P_{\tilde{X}}P_{U\mid\tilde{X}}.
\end{align} 
This formulation captures several practical situations. For example, the logged data may contain summaries of past network measurements obtained by averaging or pooling over time windows, whereas the controller may act on fine-grained instantaneous measurements \cite{LACAVA2025111342,abdalla2022toward,Haohuang}. The controller may also rely on internal states that are not exposed through standard interfaces, use protected attributes unavailable at higher layers, or interact with other applications that use information contained in $U$ but absent from $\tilde{X}$.

The telemetry available for counterfactual analysis consists of two independent datasets collected under the observational and randomized regimes. The observational dataset contains $n_{\mathrm{obs}}$ independent and identically distributed samples from the observational regime,
\begin{align}
    \label{eq:obs_dataset}
    \mathcal{D}^{\mathrm{obs}}
    =
    \bigl\{
        (\tilde{X}_i,
        A_i,
        Y_i,
        S_i)
    \bigr\}_{i=1}^{n_{\mathrm{obs}}},
    \text{ with }
    S_i=0.
\end{align}
Similarly, the randomized dataset contains $n_{\mathrm{rnd}}$ independent and identically distributed samples from the randomized regime,
\begin{align}
    \label{eq:int_dataset}
    \mathcal{D}^{\mathrm{rnd}}
    =
    \bigl\{
        (\tilde{X}_i,
        A_i,
        Y_i,
        S_i)
    \bigr\}_{i=1}^{n_{\mathrm{rnd}}},
    \text{ with }
    S_i=1.
\end{align}
The aggregate telemetry dataset comprises $n=n_{\mathrm{obs}}+n_{\mathrm{rnd}}$ samples and is denoted as 
\begin{align}
    \label{eq:dataset}
    \mathcal{D}
    =
    \mathcal{D}^{\mathrm{obs}}
    \cup
    \mathcal{D}^{\mathrm{rnd}},
\end{align}
In the following, we focus on the practically relevant regime
$n_{\mathrm{obs}}\gg n_{\mathrm{rnd}}$, in which observational telemetry collected during normal network operation is abundant, whereas randomized experimentation is limited by operational costs, safety constraints, or service-level requirements.

\subsection{Problem Definition}
\label{subsec:problem_definition}

Using the aggregate telemetry dataset $\mathcal{D}$, the goal of reliable counterfactual KPI estimation is to answer what-if questions of the following form: given a logged network state $\tilde{X}=\tilde{x}$, what KPI would have been observed had the controller selected an action $a\in\mathcal{A}$?

For a given target action $a\in\mathcal{A}$, we answer such counterfactual queries by constructing a prediction set
\begin{align}
    \label{eq:prediction_set_mapping}
    \Gamma_a(\tilde{X},\mathcal{D})
    \subseteq
    \mathcal{Y}
\end{align}
that includes the potential KPI $Y(a)$ that would be observed if action $a$ were selected under the logged context $\tilde{X}$. For a user-defined miscoverage level $\alpha\in(0,1)$, the reliability is expressed via the marginal coverage guarantee
\begin{align}
    \label{eq:marginal_coverage_potential_multi_action}
    \Pr\bigl[
        Y(a)\in\Gamma_a(\tilde{X},\mathcal{D})
    \bigr]
    \geq
    1-\alpha,
\end{align}
where the probability is taken jointly over the dataset $\mathcal{D}$ and an independent test unit distributed according to
\begin{align}
    \label{eq:target_test_distribution}
    P_{\tilde{X},U,Y(a)}
    =
    P_{\tilde{X}}
    P_{U\mid\tilde{X}}
    P_{Y(a)\mid\tilde{X},U}.
\end{align}
Condition \eqref{eq:marginal_coverage_potential_multi_action} is a marginal coverage guarantee ensuring that the prediction set $\Gamma_a(\tilde{X},\mathcal{D})$ contains the target potential KPI $Y(a)$ with probability at least $1-\alpha$.

\begin{figure}
    \centering
    \includegraphics[width=0.5
    \textwidth]{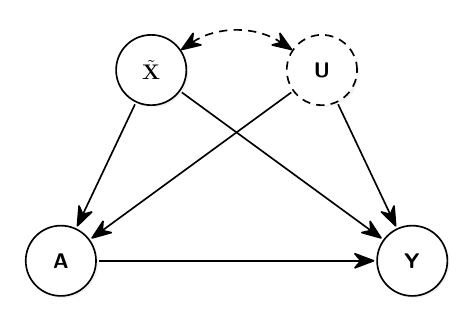}
    \caption{Causal structure under hidden confounding. In the observational regime, action $A$ is selected using the full context $X=(\tilde{X},U)$, while the unobserved component $U$ may also affect the potential KPIs ${Y(a)}_{a\in\mathcal{A}}$. Conditioning only on the logged state $\tilde{X}$ therefore does not, in general, block the dependence between $A$ and the potential outcomes. Randomized assignment removes this dependence by making $A$ independent of the context.}
    \label{fig:dag}
\end{figure}

It is important to highlight that, in the presence of hidden confounding, without additional assumptions, observational telemetry  $\mathcal{D}^{\mathrm{obs}}$ alone does not
identify the target distribution $P_{Y(a)\mid\tilde{X}}$ and therefore
cannot, in general, support nontrivial distribution-free prediction sets
with the desired target-population coverage \eqref{eq:marginal_coverage_potential_multi_action}. As illustrated in Figure~\ref{fig:dag}, the unobserved variable $U$ may affect both the selected action $A$ and the potential outcomes $\{Y(a)\}_{a\in\mathcal{A}}$. Therefore, conditional ignorability with respect to the logged state,
\begin{align}
    \label{eq:strong_ignorability}
    \{Y(a)\}_{a\in\mathcal{A}}
    \perp
    A
    \mid
    \tilde{X},S=0,
\end{align}
need not hold under the observational regime~\cite{chen2024}. As a result, the observational outcome distribution and target potential-outcome distribution are not equal in general,
\begin{align}
    \label{eq:obs_conditional_distribution}
    P_{Y\mid \tilde{X}=\tilde{x},A=a,S=0}
    \neq
    P_{Y(a)\mid \tilde{X}=\tilde{x}}.
\end{align}
It follows that observational samples assigned to action $a$ may not be representative of the potential KPI distribution for the target population.

This motivates the use of randomized telemetry $\mathcal{D}^{\mathrm{rnd}}$. In fact, under the randomized regime, the action $A$ is independent of the decision context $X=(\tilde{X},U)$ and
\begin{align}
    \label{eq:interventional_distribution_identity}
    P_{Y\mid \tilde{X}=\tilde{x},A=a,S=1}
    =
    P_{Y(a)\mid \tilde{X}=\tilde{x}}.
\end{align}

Note that the coverage guarantee in \eqref{eq:marginal_coverage_potential_multi_action} can be satisfied by a counterfactual analysis procedure that always returns the entire KPI space $\mathcal{Y}$. Despite providing coverage, however, such a procedure is uninformative. Therefore, the quality of the counterfactual prediction sets $\Gamma_a(\tilde{X},\mathcal{D})$ is measured by their average size. To define this size for different types of KPIs, let $\mu$ be a reference measure on the KPI space $\mathcal{Y}$; for instance, $\mu$ may be the Lebesgue measure for continuous KPIs and the counting measure for discrete KPIs. For each action $a\in\mathcal{A}$, we define the inefficiency of the prediction set $\Gamma_a(\cdot)$ as
\begin{align}
    \label{eq:prediction_set_efficiency}
    \Phi(\Gamma_a)
    =
    \mathbb{E}\Bigl[
        \mu\bigl(\Gamma_a(\tilde{X},\mathcal{D})\bigr)
    \Bigr],
\end{align}
where the expectation is taken over the calibration data $\mathcal{D}$ and over an independent test unit sampled according to \eqref{eq:target_test_distribution}. Among procedures satisfying the coverage guarantee in \eqref{eq:marginal_coverage_potential_multi_action}, prediction sets with smaller values of $\Phi(\Gamma_a)$ correspond to more informative counterfactual answers, because they restrict the plausible values of $Y(a)$ to a smaller region of $\mathcal{Y}$.

Finally, we note that potential-outcome prediction sets
$\{\Gamma_a(\tilde{X},\mathcal{D})\}_{a\in\mathcal{A}}$ can also be used to
compare the effect of different actions. For two candidate actions $a,b\in\mathcal{A}$ and a
logged context $\tilde{X}$, the individual treatment effect \cite{hernan2010causal}, or action effect, is
defined as
\begin{align}
    \label{eq:action_effect}
    \Delta(a,b)
    =
    Y(a)-Y(b),
\end{align}
which quantifies the KPI change that would result from selecting action $a$
instead of action $b$. Given prediction sets for $Y(a)$ and $Y(b)$, a natural
prediction set for $\Delta(a,b)$ is obtained by taking the Minkowski difference
of the two sets
\begin{align}
    \label{eq:action_effect_set}
    \Gamma_{\Delta(a,b)}(\tilde{X},\mathcal{D})
    =
    \left\{
        y_a-y_b
        :
        y_a\in\Gamma_a(\tilde{X},\mathcal{D}),
        \;
        y_b\in\Gamma_b(\tilde{X},\mathcal{D})
    \right\}.
\end{align}
Indeed, if $\Gamma_a(\tilde{X},\mathcal{D})$ and
$\Gamma_b(\tilde{X},\mathcal{D})$ satisfy
\eqref{eq:marginal_coverage_potential_multi_action} for actions $a$ and
$b$, respectively, then, by the union bound,
\begin{align}
    \Pr\bigl[
        \Delta(a,b)\in
        \Gamma_{\Delta(a,b)}(\tilde{X},\mathcal{D})
    \bigr]
    \geq
    1-2\alpha.
\end{align}
Thus, prediction sets for individual potential KPIs naturally induce valid
prediction sets for action effects, enabling the comparison of candidate control
decisions under uncertainty \cite{lei2021}.

\section{Background}
\label{sec:background}
In this section, we review the statistical tools underlying the proposed methodology. First, we introduce Conformal Counterfactual KPI Estimation (CCKE)~\cite{hou2025if}, which uses weighted conformal prediction (WCP)~\cite{tibshirani2019,lei2021} to construct prediction sets for potential KPIs under conditional ignorability.  We then review General Synthetic-Powered Inference (GESPI) \cite{bashari2025gespi}, a general framework for statistical inference that combines abundant, potentially biased data with a limited amount of unbiased data to improve efficiency while limiting the worst-case degradation in
statistical validity.
\subsection{Conformal Counterfactual KPI Estimation}
\label{subsec:wcp-no-confounding}
Conformal Counterfactual KPI Estimation (CCKE) \cite{hou2025if} is a counterfactual inference method that targets the coverage requirement \eqref{eq:marginal_coverage_potential_multi_action} under the assumption of no hidden confounding. For the model introduced in Section~\ref{sec:problem-definition}, this assumption is verified when all variables influencing action selection are observed in the logged telemetry, i.e., $X=\tilde{X}$, so that the strong ignorability condition in~\eqref{eq:strong_ignorability} holds. As shown next, under this condition, CCKE exploits WCP to construct prediction sets for potential outcomes that satisfy the marginal coverage guarantee in~\eqref{eq:marginal_coverage_potential_multi_action}. 
For the remainder of this subsection, we fix an arbitrary target action $a\in\mathcal{A}$ and, to simplify notation, we suppress the dependence on $a$ whenever it is clear from the context.

Define the subset of the observational telemetry samples $\mathcal{D}^{\rm obs}$ for which the selected action is $a$ as
\begin{align}
 \label{eq:action_specific_dataset}
    \mathcal{D}_{a}
    =
    \left\{
        (\tilde{X}_i,Y_i)
        :
        (\tilde{X}_i,A_i,Y_i,S_i)\in\mathcal{D}^{\rm obs},
        A_i=a
    \right\}.
\end{align}

CCKE first partitions $\mathcal{D}_{a}$ into a training set $\mathcal{D}^{\mathrm{tr}}_{a}$ and a calibration set $\mathcal{D}^{\mathrm{cal}}_{a}$. The training set $\mathcal{D}^{\mathrm{tr}}_{a}$ is used to fit lower and upper conditional quantile regressors, $\hat{q}_{\alpha/2}(\tilde{X})$ and $\hat{q}_{1-\alpha/2}(\tilde{X})$, which estimate the $\alpha/2$ and $1-\alpha/2$ conditional quantiles of $Y(a)$ given $\tilde{X}$, respectively. For a context $\tilde{X}$, these regressors define the uncalibrated interval for the KPI $Y(a)$ as
\begin{equation}
     \tilde{\Gamma}_a(\tilde{X})
    =
    \left[
        \hat{q}_{\alpha/2}
        \bigl(\tilde{X}\bigr),
        \;
        \hat{q}_{1-\alpha/2}
        \bigl(\tilde{X}\bigr)
    \right].
    \label{eq:uncalibrated_interval}
\end{equation}

In general, the set predictor $\tilde{\Gamma}_a(\cdot)$ in~\eqref{eq:uncalibrated_interval} is not calibrated and does not satisfy the desired marginal coverage guarantee. Therefore, CCKE uses the calibration set $\mathcal{D}^{\mathrm{cal}}_{a}$ to adjust the predictor $\tilde{\Gamma}_a(\cdot)$ by a data-dependent widening parameter $\gamma_{a}(\tilde{x})$.

Let $\mathcal{I}^{\rm cal}_a$ denote the index set of the samples in $\mathcal{D}^{\rm cal}_a$ and for each calibration 
sample $i \in \mathcal{I}^{\rm cal}_a$, define the nonconformity score
\begin{equation}
        V_i=V(\tilde{X}_i,Y_i)
    =
    \max
    \Bigl\{
        \hat{q}_{\alpha/2}
        \bigl(\tilde{X}_i\bigr)
        -
        Y_i,
        Y_i
        -
        \hat{q}_{1-\alpha/2}
        \bigl(\tilde{X}_i\bigr)
    \Bigr\}.
    \label{eq:nonconformity_score_cal}
\end{equation}

The nonconformity score $V_i$ is the largest of the lower- and
upper-tail residuals. It is positive when $Y_i$ lies outside the
uncalibrated interval $\tilde{\Gamma}_a(\tilde{X}_i)$ and may be negative when $Y_i$ lies inside it. 

Based on the calibration nonconformity scores $\{V_i\}_{i \in \mathcal{I}^{\rm cal}_a}$, 
CCKE sets the adjustment factor as the $(1-\alpha)$-quantile of a weighted empirical 
distribution of the calibration scores. The weighting accounts for the covariate shift between the calibration and test 
distributions. Specifically, the calibration covariates in $\mathcal{D}^{\mathrm{cal}}_{a}$ are distributed according to the observational law $P_{\tilde{X}\mid A=a,S=0}$, whereas the guarantee in~\eqref{eq:marginal_coverage_potential_multi_action} is given for a test covariate distributed according to $P_{\tilde{X}}$. CCKE corrects this distributional mismatch using importance weights defined by the likelihood ratio
\begin{align}
    \label{eq:wcp_likelihood_ratio}
    r(\tilde{x})
    =
    \frac{
        dP_{\tilde{X}}
    }{
        dP_{\tilde{X}\mid A=a,S=0}
    }(\tilde{x})
    \propto
    \frac{1}{
        p_{\mathrm{obs}}(a\mid\tilde{X}=\tilde{x})
    }.
\end{align}

Accordingly, contexts for which action $a$ is rarely selected under the observational policy receive larger weights, compensating for their under-representation in the action-specific calibration set.
When the observational policy $p_{\mathrm{obs}}(a\mid\tilde{X})$ is unknown, the likelihood ratio in~\eqref{eq:wcp_likelihood_ratio} must be estimated from data.

For a given context $\tilde{X}$,  the likelihood ratios are normalized jointly over the calibration samples and the test point, and each calibration sample $i\in\mathcal{I}^{\mathrm{cal}}_a$ is assigned the weight
\begin{equation}
     w_i(\tilde{X})
    =
    \frac{
        r(\tilde{X}_i)
    }{
        \sum_{j\in\mathcal{I}^{\mathrm{cal}}_a}
        r(\tilde{X}_j)
        +
        r(\tilde{X})
    },
    \label{eq:normalized_weights_cal}
\end{equation}
while the test point is assigned the weight
\begin{equation}
      w(\tilde{X})
    =
    \frac{
        r(\tilde{X})
    }{
        \sum_{j\in\mathcal{I}^{\mathrm{cal}}_a}
        r(\tilde{X}_j)
        +
        r(\tilde{X})
    }.
    \label{eq:normalized_weights_test}
\end{equation}

These weights define the weighted empirical distribution of the nonconformity 
scores
\begin{equation}
     \hat{P}_{V}^{\tilde{X}}
    =
    \sum_{i\in\mathcal{I}^{\mathrm{cal}}_a}
    w_i(\tilde{X})\,
    \delta_{V_i}
    +
    w(\tilde{X})\,
    \delta_{+\infty},
    \label{eq:weighted_empirical}
\end{equation}

where $\delta_v$ denotes the Dirac delta function at $v$, and the additional mass 
at $\infty$ represents the unknown nonconformity score of the test point. 

The correction $\gamma_a(\tilde{X})$ is chosen as the $(1-\alpha)$-quantile of  the distribution~\eqref{eq:weighted_empirical}
\begin{equation}
    \gamma_a(\tilde{X})=
     Q_{1-\alpha}
    \left(
        \hat{P}_{V}^{\tilde{X}}
    \right),
    \label{eq:correction}
\end{equation}
where, for a distribution $P$, $Q_{\tau}(P)$ denotes the $\tau$-quantile of $P$. 

For a network context $\tilde{X}$ and the target action $a$, CCKE outputs the prediction set for the potential outcome $Y(a)$
\begin{equation}
    \Gamma^{\text{CCKE}}_{a}(\tilde{X}) 
    = \Bigl[\hat{q}_{\alpha/2}(\tilde{X}) - \gamma_a(\tilde{X}),\;
             \hat{q}_{1-\alpha/2}(\tilde{X}) + \gamma_a(\tilde{X})\Bigr].
    \label{eq:calibrated_interval}
\end{equation}

Assume that the ignorability condition in~\eqref{eq:strong_ignorability}
holds, that the likelihood ratio in~\eqref{eq:wcp_likelihood_ratio} is
known exactly, and that there exists a constant $c>0$
such that
\begin{align}
    p_{\mathrm{obs}}(a\mid\tilde{X}=\tilde{x})
    \geq c
\end{align}
for all $\tilde{x}$ in the support of $P_{\tilde{X}}$. Under the sampling assumptions of Section~\ref{sec:problem-definition}, the prediction set
$\Gamma_a^{\mathrm{CCKE}}(\cdot)$ satisfies the marginal coverage guarantee
in~\eqref{eq:marginal_coverage_potential_multi_action} \cite{hou2025if}. When the likelihood ratio is estimated, the exact finite-sample guarantee
may no longer hold, and the coverage error depends on the quality of the
ratio estimate \cite{lei2021}.

As shown in the motivating example in Figure~\ref{fig:motivating_example}, in the presence of hidden confounding, the ignorability condition in~\eqref{eq:strong_ignorability} is not satisfied in general, and the CCKE prediction set in~\eqref{eq:calibrated_interval} is no longer guaranteed to satisfy the marginal coverage requirement. A direct alternative is to apply CCKE to the randomized telemetry $\mathcal{D}^{\mathrm{rnd}}$, which follows the target potential-outcome distribution. However, because $n_{\mathrm{rnd}}\ll n_{\mathrm{obs}}$, the resulting prediction sets may be inefficient.

\subsection{General Synthetic-Powered Inference}
\label{subsec:gespi}

General Synthetic-Powered Inference (GESPI) \cite{bashari2025gespi} is a framework for statistically
valid inference with auxiliary data. It is motivated by settings in which
reliable real-world data from the target distribution are scarce, while large
auxiliary datasets are easier to obtain, for example from simulators or deep
generative models. Auxiliary data can improve statistical efficiency by
increasing the effective sample size, but they may also be biased or
distributionally mismatched, thereby invalidating the guarantees of the
inference procedure. GESPI addresses this issue by using the reliable data as a
validity ``guardrail'' while exploiting the auxiliary data when they are
informative.

More specifically, let $\mathcal{D}_{\rm rel}$ denote reliable real-world data
sampled from the target distribution $P$, and let $\mathcal{D}_{\rm aux}$ denote
a larger auxiliary dataset sampled from a possibly different distribution $Q$.
GESPI wraps around a base inference procedure  $\mathrm{Alg}_{\alpha}$, such as a conformal
prediction method or a hypothesis testing procedure, that is
assumed to control a target risk at level $\alpha$ when applied to data sampled
from the target distribution $P$.

To formalize this notion, let $\ell(a,Z)$ denote the loss incurred by an
inference output $a$ when evaluated on an independent target-distribution sample
$Z\sim P$. The form of $\ell$ depends on the inference task. For confidence or
prediction sets, it may correspond to the miscoverage loss; for multiple
testing, it may measure a false-discovery criterion. The validity assumption on
the base procedure is that, when applied to reliable data, its expected loss is
at most $\alpha$, i.e.,
\begin{align}
    \label{eq:base_risk_validity}
    R(\mathrm{Alg}_{\alpha},P)
    =
    \mathbb{E}
    \left[
        \ell
        \left(
            \mathrm{Alg}_{\alpha}(\mathcal{D}_{\rm rel}),
            Z
        \right)
    \right]
    \leq
    \alpha,
\end{align}
where the expectation is taken over both the reliable dataset
$\mathcal{D}_{\rm rel}$ and the independent test sample $Z\sim P$.

The main principle behind GESPI is to leverage the validity of the base
procedure while safely incorporating auxiliary data. To this end, GESPI
aggregates three outputs of the base procedure:
\begin{itemize}
    \item The reliable-data output
    $\mathrm{Alg}_{\alpha}(\mathcal{D}_{\rm rel})$, which uses only the reliable
    data at the nominal error level $\alpha$.

    \item The auxiliary-powered output
    $\mathrm{Alg}_{\alpha}(\mathcal{D}_{\rm rel}\cup\mathcal{D}_{\rm aux})$,
    which uses both the reliable and auxiliary data at the nominal error level
    $\alpha$.

    \item The guardrail output
    $\mathrm{Alg}_{\alpha+\epsilon}(\mathcal{D}_{\rm rel})$, which uses only the
    reliable data at the relaxed error level $\alpha+\epsilon$. Here,
    $\epsilon>0$ is a user-specified tolerance parameter that determines the
    maximum additional error one is willing to allow in order to benefit from
    the auxiliary data.
\end{itemize}

To combine these three outputs, GESPI uses the natural ordering of the
statistical objects returned by the base procedure. For example, if the base
procedure returns confidence or prediction sets, then one set is more
conservative than another if it contains it. If the base procedure returns
rejection regions in a hypothesis-testing problem, then one rejection region is
more aggressive than another if it contains more rejected hypotheses. This
ordering allows GESPI to aggregate the three outputs in a way that is adapted to
the inference task.

For base procedures whose outputs are prediction sets, GESPI combines the three sets as
\begin{align}
    \label{eq:GESPI_aggregation}
    \widetilde{\mathrm{Alg}}
    =
    \mathrm{Alg}_{\alpha}(\mathcal{D}_{\rm rel})
    \cap
    \left[
        \mathrm{Alg}_{\alpha}
        \left(
            \mathcal{D}_{\rm rel}
            \cup
            \mathcal{D}_{\rm aux}
        \right)
        \cup
        \mathrm{Alg}_{\alpha+\epsilon}(\mathcal{D}_{\rm rel})
    \right].
\end{align}
The outer intersection ensures that the GESPI set is never larger than the prediction set obtained using only the reliable data at level $\alpha$. The union with the guardrail set ensures that the final set always contains the reliable-data prediction set constructed at level $\alpha+\epsilon$. Consequently, the auxiliary-powered set can reduce the size of the reliable-data set only to the extent permitted by the guardrail.

Under standard loss monotonicity assumptions, GESPI output in \eqref{eq:GESPI_aggregation} satisfies the deterministic sandwich guarantee
\begin{align}
    \label{eq:gespi_sandwich}
       \mathrm{Alg}_{\alpha+\epsilon}(\mathcal{D}_{\rm rel}) 
    \subseteq
    \widetilde{\mathrm{Alg}}
    \subseteq\mathrm{Alg}_{\alpha}(\mathcal{D}_{\rm rel}).
\end{align}
and its risk satisfies
\begin{align}
    \label{eq:gespi_risk_bound}
    R(\widetilde{\mathrm{Alg}},P)
    \leq
    \alpha+\epsilon,
\end{align}
regardless of the auxiliary-data distribution $Q$.

In the following section, we specialize this principle to the telemetry setting considered in this work. In particular, we use GESPI to construct counterfactual prediction sets by safely combining abundant observational telemetry, which may be affected by hidden confounding, with a smaller amount of randomized telemetry.

\section{Confounding-Valid Counterfactual Conformal Inference}
\label{sec:gespi_counterfactual}

We now present Confounding-Valid Counterfactual Conformal Inference (CV-CCI), a general framework for counterfactual estimation of a wireless KPI under hidden confounding. As discussed in Section~\ref{sec:problem-definition}, observational telemetry is abundant but may be confounded because the controller selects actions using the full context $X=(\tilde{X},U)$, while the analyst observes only $\tilde{X}$. Consequently, CCKE applied to observational telemetry alone may fail to satisfy the marginal coverage guarantee in~\eqref{eq:marginal_coverage_potential_multi_action}. In contrast, randomized telemetry is unconfounded because actions are randomized independently of both $\tilde{X}$ and $U$, but such data are scarce because randomized experimentation may require the network to take suboptimal actions. To exploit the large observational dataset without losing the validity provided by randomized data, CV-CCI instantiates GESPI using as a base inference procedure the CCKE construction reviewed in Section~\ref{subsec:wcp-no-confounding}.

For a target action $a\in\mathcal{A}$, define the action-specific subset of randomized telemetry as
\begin{align}
    \mathcal{D}^{\mathrm{rnd}}_{a}
    &=
    \left\{
        (\tilde{X}_i,Y_i)
        :
        (\tilde{X}_i,A_i,Y_i,S_i)\in\mathcal{D}^{\mathrm{rnd}},
        \;
        A_i=a
    \right\},
\end{align}
and the action-specific subset of observational telemetry as
\begin{align}
    \mathcal{D}^{\mathrm{obs}}_{a}
    &=
    \left\{
        (\tilde{X}_i,Y_i)
        :
        (\tilde{X}_i,A_i,Y_i,S_i)\in\mathcal{D}^{\mathrm{obs}},
        \;
        A_i=a
    \right\}.
\end{align}
The observational action-specific dataset $\mathcal{D}^{\mathrm{obs}}_{a}$ is partitioned into a training set $\mathcal{D}^{\mathrm{obs},\mathrm{tr}}_{a}$ and a calibration set $\mathcal{D}^{\mathrm{obs},\mathrm{cal}}_{a}$. Let $\mathcal{I}^{\mathrm{rnd}}_a$ denote the index set of randomized samples in $\mathcal{D}^{\mathrm{rnd}}_{a}$, and let $\mathcal{I}^{\mathrm{obs},\mathrm{cal}}_a$ denote the index set of observational calibration samples in $\mathcal{D}^{\mathrm{obs},\mathrm{cal}}_{a}$.

As in CCKE, the training set $\mathcal{D}^{\mathrm{obs},\mathrm{tr}}_{a}$ is used to fit lower and upper conditional quantile regressors, $\hat{q}_{\alpha/2}(\cdot)$ and $\hat{q}_{1-\alpha/2}(\cdot)$, which define the uncalibrated interval $\tilde{\Gamma}_a(\cdot)$ in~\eqref{eq:uncalibrated_interval}. However, unlike CCKE, CV-CCI calibrates the prediction set $\tilde{\Gamma}_a(\cdot)$ by applying GESPI to the WCP construction in Section~\ref{subsec:wcp-no-confounding}, treating the randomized telemetry $\mathcal{D}^{\mathrm{rnd}}$ as the reliable dataset and the observational telemetry $\mathcal{D}^{\mathrm{obs}}$ as the auxiliary dataset. Specifically, CV-CCI constructs three WCP prediction sets corresponding to the three components of the GESPI.

For the fitted quantile regressors $\hat{q}_{\alpha/2}(\cdot)$ and $\hat{q}_{1-\alpha/2}(\cdot)$, define the calibration nonconformity score evaluated on the observational and on the randomized data as
\begin{align}
    \mathcal{V}^{\mathrm{obs}}
    &=
    \left\{
        V(\tilde{X}_i,Y_i)
        :
        i \in\mathcal{I}^{\mathrm{obs},\mathrm{cal}}_{a}
    \right\},
\end{align}
and
\begin{align}
    \mathcal{V}^{\mathrm{rnd}}
    &=
    \left\{
        V(\tilde{X}_i,Y_i)
        :
        i \in\mathcal{I}^{\mathrm{rnd}}_{a}
    \right\}.
\end{align}

CV-CCI defines two randomized-only WCP sets by adjusting the set predictor $\tilde{\Gamma}_a(\tilde{x})$ using data-dependent widening parameters computed from the weighted empirical distribution of the scores in $\mathcal{V}^{\mathrm{rnd}}$. Since under the randomized regime actions are independent of $(\tilde{X},U)$, conditioning on $A=a$ does not change the distribution of $\tilde{X}$. Hence, the likelihood ratio between the distribution of selection covariates in $\mathcal{D}^{\mathrm{rnd}}_{a}$ and the test covariate distribution is constant. It follows that, for a test context $\tilde{X}$, the normalized WCP weights are equal to $1/(|\mathcal{I}^{\mathrm{rnd}}_a|+1)$, and the resulting weighted empirical distribution of nonconformity scores for randomized telemetry is
\begin{align}
    \label{eq:int_score_distribution}
    \hat{P}^{\,\mathrm{rnd},\tilde{X}}_{V}
    =
    \frac{1}{
        |\mathcal{I}^{\mathrm{rnd}}_a|+1
    }
    \left(
        \sum_{v \in \mathcal{V}^{\mathrm{rnd}}}
        \delta_{v}
        +
        \delta_{+\infty}
    \right).
\end{align}

The empirical distribution $\hat{P}^{\,\mathrm{rnd},\tilde{X}}_{V}$ is used to define the randomized-only prediction set
\begin{align}
    \label{eq:int_alpha_set}
    \Gamma^{\mathrm{rnd}}_{a}(\tilde{X})
    =
    \left[
        \hat{q}_{\alpha/2}(\tilde{X})
        -
        \gamma^{\mathrm{rnd}}_{\alpha}(\tilde{X}),
        \;
        \hat{q}_{1-\alpha/2}(\tilde{X})
        +
        \gamma^{\mathrm{rnd}}_{\alpha}(\tilde{X})
    \right],
\end{align}
where the widening parameter $ \gamma^{\mathrm{rnd}}_{\alpha}(\tilde{X})$ corresponds to the $(1-\alpha)$-quantile of $\hat{P}^{\,\mathrm{rnd},\tilde{X}}_{V}$. 
Similarly, for a guardrail parameter $\epsilon>0$ such that $\alpha+\epsilon<1$, the guardrail set predictor is defined as \begin{align}
    \label{eq:int_alpha_eps_set}
    \Gamma^{\mathrm{guard}}_{a}(\tilde{X})
    =
    \left[
        \hat{q}_{\alpha/2}(\tilde{X})
        -
        \gamma^{\mathrm{rnd}}_{\alpha+\epsilon}(\tilde{X}),
        \;
        \hat{q}_{1-\alpha/2}(\tilde{X})
        +
        \gamma^{\mathrm{rnd}}_{\alpha+\epsilon}(\tilde{X})
    \right],
\end{align}
where the widening parameter $ \gamma^{\mathrm{rnd}}_{\alpha+\epsilon}(\tilde{X})$ corresponds to the $(1-\alpha-\epsilon)$-quantile of  $\hat{P}^{\,\mathrm{rnd},\tilde{X}}_{V}$.

The third set computed by CV-CCI is the auxiliary-powered pooled set. This set corresponds to a WCP set obtained by calibrating the base set predictor $\tilde{\Gamma}_a(\tilde{X})$ using a weighted empirical distribution of both randomized and observational calibration scores. Observational calibration samples are weighted using the likelihood ratio between the test distribution $P_{\tilde{X}}$ and the action-specific observational calibration distribution $P_{\tilde{X}\mid A=a,S=0}$ given by
\begin{align}
    \label{eq:obs_likelihood_ratio_hidden}
    r^{\mathrm{obs}}_a(\tilde{x})
    =
    \frac{
        dP_{\tilde{X}}
    }{
        dP_{\tilde{X}\mid A=a,S=0}
    }(\tilde{x})
    \propto
    \frac{
       1
    }{
        \Pr(A=a\mid \tilde{X}=\tilde{x},S=0)
    },
\end{align}
where $\Pr(A=a\mid \tilde{X}=\tilde{x},S=0)$ denotes the probability of selecting action $a$ under logged context $\tilde{x}$ in the observational regime. In general, the likelihood ratios in~\eqref{eq:obs_likelihood_ratio_hidden} are unknown and must be estimated from data, for example via a probabilistic classifier trained to predict the selected action $A$ from the logged state $\tilde{X}$. For randomized samples, the likelihood ratio is constant, and therefore, the pooled unnormalized weights are defined as
\begin{align}
    \label{eq:pooled_likelihood_ratio}
    \hat{r}^{\mathrm{pool}}_a(\tilde{X}_i)
    =
    \begin{cases}
        \beta,
        &
        i\in\mathcal{I}^{\mathrm{rnd}}_a,
        \\[0.4em]
        \hat{r}^{\mathrm{obs}}_a(\tilde{X}_i),
        &
        i\in\mathcal{I}^{\mathrm{obs},\mathrm{cal}}_a,
    \end{cases}
\end{align}
where $\beta>0$ is a parameter that controls the relative weight assigned to randomized telemetry samples in the pooled auxiliary distribution; in our implementation, we set $\beta=1$. Thus, for a test context $\tilde{X}$, the normalized pooled weights are
\begin{align}
    \label{eq:pooled_weights}
    w_i^{\mathrm{pool}}(\tilde{X})
    =
    \frac{
        \hat{r}_a^{\mathrm{pool}}(\tilde{X}_i)
    }{
        \sum_{j\in\mathcal{I}^{\mathrm{rnd}}_a}
        \hat{r}_a^{\mathrm{pool}}(\tilde{X}_j)
        +
        \sum_{j\in\mathcal{I}^{\mathrm{obs},\mathrm{cal}}_a}
        \hat{r}_a^{\mathrm{pool}}(\tilde{X}_j)
        +
        \beta
    },
\end{align}
for $i\in\mathcal{I}^{\mathrm{rnd}}_a\cup\mathcal{I}^{\mathrm{obs},\mathrm{cal}}_a$, and for the test covariate
\begin{align}
    \label{eq:pooled_test_weight}
    w^{\mathrm{pool}}(\tilde{X})
    =
    \frac{
        \beta
    }{
        \sum_{j\in\mathcal{I}^{\mathrm{rnd}}_a}
        \hat{r}_a^{\mathrm{pool}}(\tilde{X}_j)
        +
        \sum_{j\in\mathcal{I}^{\mathrm{obs},\mathrm{cal}}_a}
        \hat{r}_a^{\mathrm{pool}}(\tilde{X}_j)
        +
        \beta
    }.
\end{align}
These weights define the pooled weighted empirical distribution
\begin{align}
    \label{eq:pooled_score_distribution}
    \hat{P}^{\,\mathrm{pool},\tilde{X}}_{V}
    =
    \sum_{i\in\mathcal{I}_a^{\mathrm{rnd}}}
    w_i^{\mathrm{pool}}(\tilde{X})\,
    \delta_{V(\tilde{X}_i,Y_i)}
    +
    \sum_{i\in\mathcal{I}^{\mathrm{obs},\mathrm{cal}}_a}
    w_i^{\mathrm{pool}}(\tilde{X})\,
    \delta_{V(\tilde{X}_i,Y_i)}
    +
    w^{\mathrm{pool}}(\tilde{X})\,
    \delta_{+\infty}.
\end{align}
Based on the empirical distribution $\hat{P}^{\,\mathrm{pool},\tilde{X}}_{V}$, CV-CCI computes the pooled conformal correction
\begin{align}
    \label{eq:pool_alpha_adjustment}
    \gamma^{\mathrm{pool}}_{\alpha}(\tilde{X})
    =
    Q_{1-\alpha}
    \left(
        \hat{P}^{\,\mathrm{pool},\tilde{X}}_{V}
    \right),
\end{align}
and the corresponding auxiliary-powered prediction set
\begin{align}
    \label{eq:pool_alpha_set}
    \Gamma^{\mathrm{pool}}_{a}(\tilde{X})
    =
    \left[
        \hat{q}_{\alpha/2}(\tilde{X})
        -
        \gamma^{\mathrm{pool}}_{\alpha}(\tilde{X}),
        \;
        \hat{q}_{1-\alpha/2}(\tilde{X})
        +
        \gamma^{\mathrm{pool}}_{\alpha}(\tilde{X})
    \right].
\end{align}
This pooled set may be substantially more efficient than the randomized-only sets $\Gamma^{\mathrm{rnd}}_{a}$ and $\Gamma^{\mathrm{guard}}_{a}$ when the observational data are informative. Nevertheless, because the observational component may remain confounded through $U$, the prediction set in~\eqref{eq:pool_alpha_set} alone does not provide the marginal coverage guarantee in~\eqref{eq:marginal_coverage_potential_multi_action}.

For this reason, CV-CCI obtains the final prediction set by applying the GESPI aggregation rule as
\begin{align}
    \label{eq:cvcci_prediction_set}
    \Gamma^{\mathrm{CV\text{-}CCI}}_a(\tilde{X})
    =
    \Gamma^{\mathrm{rnd}}_{a}(\tilde{X})
    \cap
    \left(
        \Gamma^{\mathrm{pool}}_{a}(\tilde{X})
        \cup
        \Gamma^{\mathrm{guard}}_{a}(\tilde{X})
    \right).
\end{align}
Following the GESPI rationale, the guardrail set $\Gamma^{\mathrm{guard}}_{a}(\tilde{X})$ limits the deterioration in marginal coverage to the user-specified tolerance $\epsilon$.

\begin{proposition}[Finite-sample coverage of CV-CCI]
\label{prop:cvcci_coverage}
Fix an action $a\in\mathcal A$ and let $\alpha,\epsilon>0$ satisfy
$\alpha+\epsilon<1$. Under the setting and sampling assumptions of
Section~\ref{sec:problem-definition}, the CV-CCI prediction set in
\eqref{eq:cvcci_prediction_set} satisfies
\begin{align}
    \label{eq:cvcci_sandwich}
    \Gamma^{\mathrm{guard}}_a(\tilde{x})
    \subseteq
    \Gamma^{\mathrm{CV\text{-}CCI}}_a(\tilde{x})
    \subseteq
    \Gamma^{\mathrm{rnd}}_a(\tilde{x})
\end{align}
for every logged context $\tilde{x}$. Consequently,
\begin{align}
    \label{eq:cvcci_prop_coverage}
    \Pr\!\left[
        Y(a)\in
        \Gamma^{\mathrm{CV\text{-}CCI}}_a(\tilde{X})
    \right]
    \geq
    1-\alpha-\epsilon,
\end{align}
where the probability is taken jointly over the telemetry data and an
independent test unit distributed according to
\eqref{eq:target_test_distribution}.
\end{proposition}
\begin{proof}
Since $\alpha+\epsilon>\alpha$, the randomized conformal quantiles satisfy
\[
\gamma^{\mathrm{rnd}}_{\alpha+\epsilon}(\tilde{x})
\leq
\gamma^{\mathrm{rnd}}_{\alpha}(\tilde{x}),
\]
and hence
\[
\Gamma^{\mathrm{guard}}_a(\tilde{x})
\subseteq
\Gamma^{\mathrm{rnd}}_a(\tilde{x}).
\]
Therefore, from \eqref{eq:cvcci_prediction_set},
\[
\Gamma^{\mathrm{guard}}_a(\tilde{x})
\subseteq
\Gamma^{\mathrm{CV\text{-}CCI}}_a(\tilde{x})
\subseteq
\Gamma^{\mathrm{rnd}}_a(\tilde{x}).
\]
Under the randomized regime, the action-specific calibration samples and the
test pair $(\tilde{X},Y(a))$ are exchangeable. Thus, standard split conformal
prediction gives
\[
\Pr\!\left[
Y(a)\in\Gamma^{\mathrm{guard}}_a(\tilde{X})
\right]
\geq
1-\alpha-\epsilon.
\]
The result follows from
$\Gamma^{\mathrm{guard}}_a(\tilde{X})
\subseteq
\Gamma^{\mathrm{CV\text{-}CCI}}_a(\tilde{X})$.
\end{proof} 
The guarantee in \eqref{eq:cvcci_prop_coverage} establishes that CV-CCI safely
combines abundant observational telemetry with scarce randomized telemetry
while preserving the validity of counterfactual KPI estimation, up to the
user-specified tolerance $\epsilon$, under arbitrary hidden confounding and
arbitrary errors in the estimated observational propensity scores.

\section{Experiments}

\label{sec:experiments}
In this section, we evaluate CV-CCI and compare it against state-of-the-art counterfactual conformal inference baselines on two representative RAN control tasks: medium access control (MAC)-layer scheduling~\cite{yaacoub2011survey} and handover~\cite{ahmed2013enabling}. The code for reproducing the experiments is available at \href{https://github.com/abdessamedqchohi/Confounding-Valid-Conformal-Counterfactual-Inference}{https://github.com/abdessamedqchohi/Confounding-Valid-Conformal-Counterfactual-Inference}.

\subsection{Baselines}
\label{subsec:experimental-baselines}

We compare CV-CCI to the following baselines:

\begin{itemize}
    \item \textbf{CCKE} \cite{hou2025if}.
    CCKE is the counterfactual conformal KPI estimator reviewed in
    Section~\ref{subsec:wcp-no-confounding}. For each action, it applies
    WCP to action-specific observational telemetry,
    using inverse propensity scores to correct for selection bias in
    the logged data. Its validity guarantee hinges on the conditional
    ignorability assumption, and therefore may not satisfy the coverage requirement \eqref{eq:marginal_coverage_potential_multi_action} under hidden confounding.

    \item \textbf{ Weighted Split Conformal Prediction with Density-Ratio Estimation (wSCP-DR)} \cite{chen2024}.
    wSCP-DR addresses hidden confounding by estimating the  ratio between the densities of randomized and observational
    samples assigned to action $a$, i.e.
    \begin{equation}
        \label{eq:ratio_obs_int}
        r(\tilde{x},y)
        =
        \frac{
            dP_{\tilde{X},Y|S=1,A=a}
        }{
            dP_{\tilde{X},Y|S=0,A=a}
        }(\tilde{x},y).
    \end{equation}
    This ratio \eqref{eq:ratio_obs_int} is estimated by training a probabilistic classifier to distinguish observational from randomized samples. wSCP-DR then applies WCP to both the observational and randomized telemetry, correcting for hidden confounding by reweighting the samples using the estimated density ratios.

      We evaluate the two-stage variants proposed in \cite{chen2024}: \textbf{wSCP-DR Inexact} and \textbf{wSCP-DR Exact}. Both methods first construct WCP intervals by combining observational and randomized data through density-ratio weighting. They then learn an interval predictor by regressing the lower and upper interval endpoints on the logged context, enabling prediction intervals to be generated for new samples. \textbf{wSCP-DR Inexact} directly uses the regressed endpoints as the final prediction interval, making it computationally efficient but without finite-sample coverage guarantees. \textbf{wSCP-DR Exact} instead applies a second split-conformal calibration step to the learned interval predictor, restoring finite-sample validity. This additional calibration requires reserving part of the already limited randomized data, which may reduce statistical efficiency and lead to wider prediction intervals.

    \item \textbf{Guardrail}.
    This baseline corresponds to the guardrail component of
    CV-CCI. The base quantile regressors are trained using observational
    telemetry, while randomized samples are used for calibration at an error level
    $\alpha+\epsilon$. The guardrail predictor satisfies the marginal coverage requirement at the level 
   $ 1-\alpha-\epsilon$
    even under hidden confounding. This baseline isolates the benefit obtained
    by the auxiliary observational component of CV-CCI.
\end{itemize}

\subsection{Medium Access Control-Layer Scheduling}
\label{subsec:scheduling}

\subsubsection{Setup}
\label{subsubsec:scheduling-setup}

\begin{figure}
    \centering
    \includegraphics[width=0.8\textwidth]{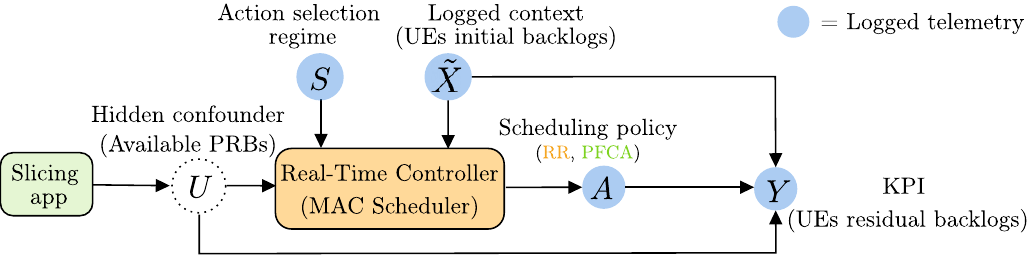}
    \caption{At each scheduling frame, the MAC scheduler selects an action $A\in\{\mathrm{RR},\mathrm{PFCA}\}$ based on the full context  $X=(\tilde{X},U)$, where $\tilde{X}$ contains the logged UE backlogs,  CQIs, and packet delays, and $U$ represents the instantaneous PRB budget assigned to the slice by a slicing application. After the selected scheduling policy is applied, the network produces the residual-backlog KPI $Y$. The PRB budget is available to the scheduler, but is not recorded in the telemetry used for counterfactual inference. Because it may affect both the action selection and the resulting KPI, it acts as a hidden confounder.}
    \label{fig:exp1}
\end{figure}

We consider a resource-allocation problem \cite{yaacoub2011survey}, similar to that studied in \cite{hou2025if}, for an orthogonal frequency-division multiplexing system with $K$ UEs. At each resource-allocation frame, the controller selects one of two scheduling policies,
\begin{equation}
    A\in\left\{\mathrm{RR},\mathrm{PFCA}\right\},
\end{equation}
where $\mathrm{RR}$ denotes round-robin scheduling and $\mathrm{PFCA}$ denotes proportional-fair channel-aware scheduling.

The complete decision context comprises the UE backlog sizes $\{b_k\}_{k=1}^{K}$, the average packet delays in the backlogs $\{d_k\}_{k=1}^{K}$, the channel quality indicators (CQIs) $\{c_k\}_{k=1}^{K}$, and the instantaneous number of available physical resource blocks $N_{\mathrm{PRB}}$. Accordingly,
\begin{equation}
    X
    =
    \left[
        b_{1:K},
        c_{1:K},
        d_{1:K},
        N_{\mathrm{PRB}}
    \right]
    \in
    \mathbb{R}^{3K+1}.
    \label{eq:scheduling-full-context}
\end{equation}
The logged context contains the UE-level measurements, but not the instantaneous PRB budget
\begin{equation}
    \tilde{X}
    =
    \left[
        b_{1:K},
        c_{1:K},
        d_{1:K}
    \right]
    \in
    \mathbb{R}^{3K}.
    \label{eq:scheduling-logged-context}
\end{equation}
Hence, the hidden component of the decision context is the total number of PRBs available for allocation, i.e., $U=N_{\mathrm{PRB}}$. As illustrated in Figure \ref{fig:exp1}, this setting models a hierarchical RAN architecture in which a slice-level resource-management application determines the available PRB budget without exposing it through the scheduler's telemetry interface. The PRB budget nevertheless affects both the preferred scheduling policy and the resulting KPI, and may therefore act as a hidden confounder.

The target KPI for counterfactual inference is the vector of residual UE backlog sizes at the end of the scheduling frame under the alternative scheduling policy $A\in\left\{\mathrm{RR},\mathrm{PFCA}\right\}$. Specifically, let $Y_k(A)$
denote the residual backlog of UE $k$ at the end of the scheduling frame under policy $A$. The inferential target is then the vector
\begin{align}
    Y(A)
    =
    \left[
        Y_1(A),\ldots,Y_K(A)
    \right]
    \in
    \mathbb{R}^{K}.
\end{align}
To construct a prediction set for the multidimensional target KPI $Y(A)$, we adopt an approach similar to that in~\cite{hou2025if}. Specifically, all methods are instantiated using a multidimensional quantile regressor, and the nonconformity score is defined as the maximum of the componentwise nonconformity scores. Calibrating this maximum score yields a prediction set with marginal simultaneous coverage for the entire KPI vector.

Under observational telemetry, the probability of selecting RR scheduling is modeled as
\begin{equation}
  p_{\mathrm{obs}}(\mathrm{RR}\mid\tilde{X},N_{\mathrm{PRB}})
  =
  \lambda\, p_U(N_{\mathrm{PRB}})
  +
  (1-\lambda)\, p_{\tilde X}(\tilde X),
  \label{eq:sched_obs_policy}
\end{equation}
where the mixture coefficient $\lambda\in[0,1]$ controls the relative influence of the unobserved PRB budget and the observed logged context.

The first component captures the dependence on the hidden context. Specifically, denoting the sigmoid function by $\sigma(\cdot)$, we define
\begin{align}
p_U(N_{\mathrm{PRB}})
=
\sigma\!\left(
N_{\mathrm{PRB}}-\bar N_{\mathrm{PRB}}
\right),
\label{eq:pU}
\end{align}
where $\bar N_{\mathrm{PRB}}$ is a reference PRB level.  Equation~\eqref{eq:pU} implies that the RR scheduler is increasingly favored as the number of available PRBs increases.

Following \cite{hou2025if}, the second component is constructed from an estimate $\hat q_k^{\mathrm{tx}}(c_k)$ of the throughput achievable by UE $k$ under RR scheduling, obtained from its CQI $c_k$. Specifically, we define
\begin{align}
p_{\tilde X}(\tilde X)
=
\sigma\!\left(
-\max_k
\bigl(
b_k-\hat q_k^{\mathrm{tx}}(c_k)
\bigr)
\right),
\label{eq:pX}
\end{align}
where $b_k-\hat q_k^{\mathrm{tx}}(c_k)$ represents the estimated residual backlog of UE $k$ under RR scheduling. Consequently, $p_{\tilde X}(\tilde X)$ decreases as the maximum residual backlog increases, reflecting a preference for PFCA when RR is expected to leave a large backlog for at least one UE.

The parameter $\lambda$ directly controls the degree of hidden confounding. When $\lambda=0$, the scheduling policy depends only on the observed context $\tilde X$, so the no-hidden-confounding assumption holds. As $\lambda$ increases, the policy relies progressively more on the unobserved PRB budget. In the limiting case $\lambda=1$, the controller ignores the logged context entirely and bases its decision solely on the hidden variable $N_{\mathrm{PRB}}$, resulting in the strongest level of hidden confounding.

Randomized telemetry is generated using a randomized policy under which RR and PFCA are chosen with the same probability, i.e.,
\begin{equation}
    p_{\mathrm{rnd}}(\mathrm{RR})
    =
    p_{\mathrm{rnd}}(\mathrm{PFCA})
    =
    \frac{1}{2}.
    \label{eq:scheduling-interventional-policy}
\end{equation}

\subsubsection{Results}
\label{subsubsec:scheduling-results}

We evaluate the benchmarked methods using telemetry generated by the Nokia Wireless Suite Simulator \cite{nokia-wireless-suite}. Specifically, we consider a scenario with $K=8$ UEs, where the number of available PRBs is uniformly distributed as $N_{\mathrm{PRB}} \sim \mathcal{U}\{4,\ldots,11\}$. For each UE $k\in\{1,\ldots,K\}$, the CQI $c_k$ is generated according to the standard 3GPP pathloss model specified in Table~7.2.3-1 of TS~36.213 Rel-11 \cite{3gpp}. The initial backlog $b_k$ is independently sampled from a geometric distribution with mean $20\,000$ bits, while the initial delay $d_k$, measured in transmission time intervals (TTIs), is sampled uniformly from $\{0,\ldots,50\}$.

We consider an observational telemetry dataset $\mathcal{D}^{\rm obs}$ containing $n_{\rm obs}=7000$ samples. Of these, $n^{\rm tr}=5000$ samples are used to train the conditional quantile regressors and propensity models, while the remaining $n^{\rm cal}_{\rm obs}=2000$ samples are reserved for calibration. In addition, we consider a limited randomized telemetry dataset $\mathcal{D}^{\rm rnd}$ comprising $n_{\rm rnd}=50$ samples. Using these datasets, we evaluate all methods with a target miscoverage level of $\alpha=0.1$ and a guardrail parameter of $\epsilon=0.025$.

\begin{figure}[]
    \centering
    \includegraphics[width=0.99\textwidth]{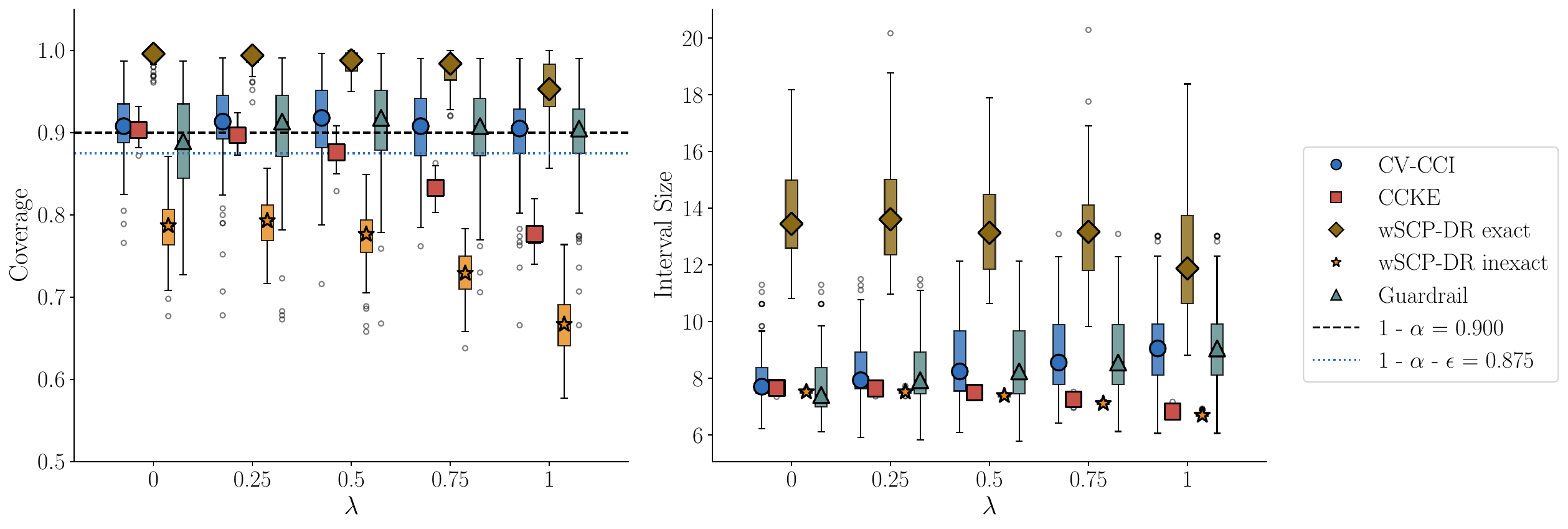}
    \caption{Empirical coverage (left) and prediction-set efficiency (right) for the RR potential outcome as a function of the confounding strength $\lambda$. The dashed line indicates the nominal coverage level $1-\alpha=0.90$, while the dotted line indicates the guardrail level $1-\alpha-\epsilon=0.875$. Boxplots summarize the results over ${n_{\rm folds}} = 100$ folds.}
    \label{fig:sched-rr}
\end{figure}

Figure~\ref{fig:sched-rr} reports empirical coverage and efficiency of the prediction sets for the RR potential outcome produced by the benchmarked methods as the confounding strength varies over $ \lambda\in\left\{ 0, 0.25,  0.5,  0.75,1 \right\}$. The empirical coverage is measured as the fraction of the $n^{\mathrm{te}}=1000$ test instances for which the residual backlogs $\{Y_i^k(a)\}^K_{k=1}$  lie within their respective predicted intervals $\{\Gamma_a^k(\cdot)\}^K_{k=1}$. For action $a$, it is defined as
\begin{equation}
    \widehat{\operatorname{Cov}}_a
    =
    \frac{1}{n^{\mathrm{te}}}
    \sum_{i=1}^{n^{\mathrm{te}}}
    \prod_{k=1}^{K}
    \mathds{1}
    \left\{
        Y_i^k(a)
        \in
        \Gamma_a^k(\tilde{X}_i)
    \right\}.
    \label{eq:scheduling-empirical-coverage}
\end{equation}
Efficiency is measured by the average size of the predicted backlog intervals $\{\Gamma_a^k(\cdot)\}^K_{k=1}$, i.e.,
\begin{equation}
    \widehat{\Phi}_a
    =
    \frac{1}{n^{\mathrm{te}}}
    \sum_{i=1}^{n^{\mathrm{te}}}
    \frac{1}{K}
    \sum_{k=1}^{K}
    \mu
    \left(
        \Gamma_a^k(\tilde{X}_i)
    \right),
    \label{eq:scheduling-empirical-efficiency}
\end{equation}
where $\mu(\cdot)$ denotes the Lebesgue measure. 

These results highlight the effect of hidden confounding on the benchmarked methods. When $\lambda=0$, and the decision policy does not depend on the hidden PRB budget $N_{\rm PRB}$, CCKE, wSCP-DR Exact, CV-CCI, and the Guardrail attain coverage close to or above the nominal level $1-\alpha=0.9$, whereas wSCP-DR Inexact already exhibits substantial undercoverage.

As $\lambda$ increases, the coverage of the methods operating under the no-hidden confounding assumption deteriorates markedly. In particular, the empirical coverage of CCKE decreases from approximately $0.91$ at $\lambda=0$ to approximately $0.78$ at $\lambda=1$. The coverage of wSCP-DR Inexact follows a similar trend, decreasing from approximately $0.79$ to approximately $0.67$. Although both methods produce comparatively narrow prediction sets, their efficiency is accompanied by increasingly severe undercoverage. Hence, the reduction in interval width should not be interpreted as an improvement in statistical efficiency.

By contrast, CV-CCI, Guardrail, and wSCP-DR Exact maintain empirical coverage above the nominal level for all values of $\lambda$. However, while the coverage achieved by CV-CCI and Guardrail remains close to the nominal level, wSCP-DR Exact exhibits noticeable overcoverage, resulting in substantially wider prediction sets. This loss in efficiency stems from the additional data split required to restore finite-sample validity. The resulting reduction in the effective calibration sample is particularly detrimental in the setting considered here, where randomized data are severely limited. Overall, these results show that CV-CCI and the Guardrail achieve the most favorable balance between robustness to hidden confounding and prediction-set efficiency.

To further assess the benefit of incorporating observational telemetry, Table~\ref{tab:sched-fails-pfca} compares the empirical failure rates of CV-CCI and Guardrail for the PFCA potential outcome. The empirical failure rate is defined as the fraction of the $n_{\rm folds}=100$ folds in which empirical coverage falls below the guardrail level $1-\alpha-\epsilon=0.875$, i.e.,
\begin{equation}
    \widehat{F}_a
    =
    \frac{1}{n_{\rm folds}}
    \sum_{i=1}^{n_{\rm folds}}
    \mathds{1}
    \left\{
        \widehat{\operatorname{Cov}}_{a,i}
        <
        1-\alpha-\epsilon
    \right\},
    \label{eq:fold-failure-rate}
\end{equation}
where $\widehat{\operatorname{Cov}}_{a,i}$ denotes the empirical coverage for action $a$ in fold $i$. Whereas average coverage summarizes performance across all data splits, the empirical failure rate measures how often a method fails to attain the prescribed guardrail coverage level.  This distinction is practically relevant because, although marginal coverage cannot be guaranteed for every realization of the limited randomized calibration dataset, practitioners would prefer a method that consistently attains the desired coverage level across different realizations of the calibration dataset.

\begin{table}[]
    \centering
    \caption{Empirical failure rates of CV-CCI and Guardrail for the PFCA potential outcome.}
    \label{tab:sched-fails-pfca}
    \begin{tabular}{lcc}
        \toprule
        \textbf{Confounding ($\lambda$)} & \textbf{CV-CCI} & \textbf{Guardrail}  \\
        \midrule
        0.00 & 30\% & 47\%  \\
        0.25 & 23\% & 38\%  \\
        0.50 & 24\% & 41\%  \\
        0.75 & 28\% & 48\%  \\
        1.00 & 26\% & 41\%  \\
        \bottomrule
    \end{tabular}
\end{table}

The results in Table~\ref{tab:sched-fails-pfca} show that CV-CCI consistently achieves a lower empirical failure probability than Guardrail across all considered confounding levels. In the absence of hidden confounding, i.e., at $\lambda=0$,  Guardrail falls below the coverage level $1-\alpha-\epsilon$ in $47\%$ of the folds, compared with $30\%$ for CV-CCI. Under maximum confounding, i.e., at $\lambda=1$, the corresponding failure probabilities are $41\%$ for Guardrail and $26\%$ for CV-CCI.

We next investigate the effect of the user-specified tolerance $\epsilon$, which determines the maximum additional miscoverage allowed beyond the nominal miscoverage level $\alpha$ when observational telemetry is incorporated as auxiliary data. Although the coverage guarantee in~\eqref{eq:cvcci_prop_coverage} holds for any admissible value of $\epsilon$, its choice directly affects both the empirical coverage and the width of the resulting prediction intervals.

We fix the hidden confounding strength at $\lambda=0.25$ and evaluate CV-CCI and Guardrail for $\epsilon\in\{0.01,0.025,0.05,0.075,0.1\}$, keeping all other experimental parameters unchanged. Figure~\ref{fig:sweep_eps_PFCA} reports the empirical coverage and prediction-interval efficiency for the PFCA potential outcome $Y(\mathrm{PFCA})$.

\begin{figure}[t]
    \centering
    \includegraphics[width=\columnwidth]{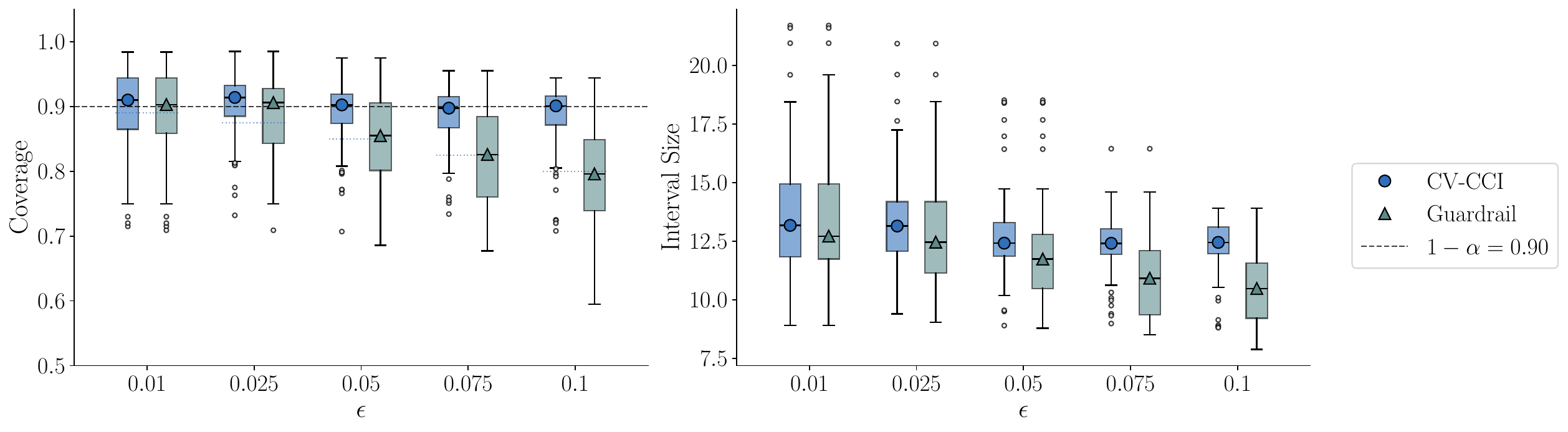}
    \caption{Empirical coverage (left) and efficiency (right) of CV-CCI and Guardrail for the PFCA potential outcome $Y(\mathrm{PFCA})$ as a function of the miscoverage tolerance $\epsilon$.}
    \label{fig:sweep_eps_PFCA}
\end{figure}

Both prediction-set methods satisfy the marginal coverage guarantee at level $1-\alpha-\epsilon$. However, CV-CCI maintains empirical coverage close to the nominal level $1-\alpha=0.90$ for all considered values of $\epsilon$, whereas the empirical coverage of Guardrail closely tracks the lower guarantee $1-\alpha-\epsilon$ and therefore decreases as $\epsilon$ increases. Moreover, the variability in the empirical coverage of Guardrail increases with $\epsilon$, as reflected by the widening interquartile range. These results show that, by incorporating observational telemetry, CV-CCI mitigates the variability arising from the limited randomized calibration sample while maintaining an empirical coverage level close to the nominal level across the full range of $\epsilon$.

\subsection{Handover}
\label{subsec:handover}

\subsubsection{Setup}
\label{subsubsec:handover-setup}

\begin{figure}[]
    \centering
    \includegraphics[width=0.8\textwidth]{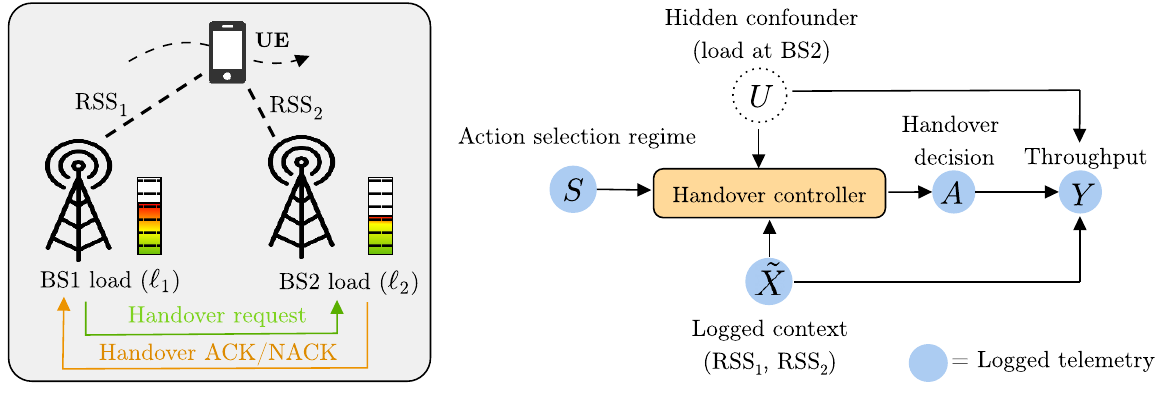}
   \caption{A serving base station (BS1) proposes a handover using the logged RSS histories $\tilde{X}=[\mathrm{RSS}_1^{-\tau},\mathrm{RSS}_2^{-\tau}]$, while a neighboring base station (BS2) accepts or rejects the request based on its instantaneous load $U=\ell_2$. The resulting action $A\in\{0,1\}$ determines whether the UE remains connected to BS1 or is handed over to BS2, and the network produces the corresponding throughput KPI $Y$. Telemetry records $(\tilde{X},A,Y)$ but not the target-cell load $U$. Because $U$ may influence both the handover decision and the resulting throughput, it acts as a hidden confounder in counterfactual KPI inference.}
    \label{fig:exp2}
\end{figure}

We consider the wireless handover scenario illustrated in Figure \ref{fig:exp2}. A UE is initially connected to a serving base station (BS1) and moves within a deployment area that also contains a neighboring base station (BS2). As the UE moves, it periodically reports the received signal strength (RSS) from both BS1 and BS2. At each reporting instant, the network determines whether a handover opportunity exists based on past RSS measurements. When such an opportunity arises, the network decides whether the UE should remain connected to BS1 ($A=0$) or be handed over to BS2 ($A=1$) based on the traffic load at BS2.

Our objective is to answer counterfactual queries about the throughput that would be achieved under either handover decision. Accordingly, for each action $a\in\{0,1\}$, we define the potential outcome $Y(a)$ as the average throughput achieved over a look-ahead horizon of $H$ TTIs if action $a$ were taken.

A handover opportunity is declared if the signal from BS2 has been stronger than that from BS1 throughout the previous $\tau$ reporting steps. Conditional on such an opportunity, the handover decision is modeled as a two-stage process. First, BS1 decides whether to initiate a handover request based on the latest $\tau$ RSS measurements from both cells. Let $\sigma(\cdot)$ denote the sigmoid function, and let $\mathrm{RSS}_1^{-\tau}=[\mathrm{RSS}_1^{t-\tau+1},\dots,\mathrm{RSS}_1^{t}]$ and $\mathrm{RSS}_2^{-\tau}=[\mathrm{RSS}_2^{t-\tau+1},\dots,\mathrm{RSS}_2^{t}]$ denote the latest $\tau$ RSS measurements from BS1 and BS2, respectively. The proposal probability is
\begin{equation}
p_{\mathrm{prop}}
\!\left(
\mathrm{RSS}_1^{-\tau},
\mathrm{RSS}_2^{-\tau}
\right)
=
\sigma\!\left(
\overline{\Delta\mathrm{RSS}}
-
m_p
\right),
\label{eq:ho-prop}
\end{equation}
where
\begin{equation}
\overline{\Delta\mathrm{RSS}}
=
\frac{1}{\tau}
\sum_{i=0}^{\tau-1}
\left(
\mathrm{RSS}_2^{t-i}
-
\mathrm{RSS}_1^{t-i}
\right)
\end{equation}
is the average RSS advantage of BS2 over BS1 during the observation window, and $m_p$ is a reference RSS margin.

If BS1 proposes a handover, BS2 decides whether to accept the request based on its instantaneous load $\ell_2$. The acceptance probability is
\begin{equation}
p_{\mathrm{acc}}(\ell_2)
=
\sigma\!\left(
\lambda
\bigl(
\tau_\ell-\ell_2
\bigr)
\right),
\label{eq:ho-accept}
\end{equation}
where $\tau_\ell$ is a reference load threshold and $\lambda$ controls the influence of the BS2 load on the acceptance decision.

Therefore, under normal network operation, the probability that a handover is executed is
\begin{equation}
p_{\mathrm{obs}}
\!\left(
A=1
\mid
\mathrm{RSS}_1^{-\tau},
\mathrm{RSS}_2^{-\tau},
\ell_2
\right)
=
p_{\mathrm{prop}}
\!\left(
\mathrm{RSS}_1^{-\tau},
\mathrm{RSS}_2^{-\tau}
\right)
p_{\mathrm{acc}}(\ell_2).
\end{equation}

We assume that BS1 has access to RSS measurements required for handover evaluation, but not to the instantaneous load of BS2. This reflects practical operational deployments, where neighboring cells exchange measurement reports to support mobility management, while detailed scheduler state and resource utilization remain local to each base station. As a result, the logged context available to the BS1 is
\begin{equation}
\widetilde X
=
\bigl[
\mathrm{RSS}_1^{-\tau},
\mathrm{RSS}_2^{-\tau}
\bigr]
\in
\mathbb R^{2\tau},
\label{eq:ho-logged-context}
\end{equation}
whereas the handover decision is made using the full context
\begin{equation}
X
=
\bigl[
\mathrm{RSS}_1^{-\tau},
\mathrm{RSS}_2^{-\tau},
\ell_2
\bigr]
\in
\mathbb R^{2\tau+1}.
\label{eq:ho-full-context}
\end{equation}
Therefore, the load $\ell_2$ acts as the hidden confounder, namely $U=\ell_2$, while the parameter $\lambda$ controls the strength of hidden confounding. When $\lambda=0$, the acceptance decision is randomized and independent of $\ell_2$, so the no-hidden-confounding assumption holds. As $\lambda$ increases, the acceptance decision depends increasingly on the unobserved load of BS2, resulting in progressively stronger hidden confounding.

Finally, we assume that randomized telemetry is generated by randomizing the handover decision independently of the observed and hidden contexts, i.e.,
\begin{equation}
p_{\mathrm{rnd}}(A=0)
=
p_{\mathrm{rnd}}(A=1)
=
\frac{1}{2}.
\label{eq:ho-interventional-policy}
\end{equation}

\subsubsection{Results}
\label{subsubsec:handover-results}

We evaluate the benchmarked methods on the wireless handover scenario described in Section~\ref{subsubsec:handover-setup} using SionnaRT ray tracing for the physical- and link-layer simulation \cite{sionna}. We consider an RSS history of length $\tau=5$ and a prediction horizon of $H=10$ TTIs. The observational telemetry dataset comprises $n_{\rm obs}=7000$ samples, of which $n^{\rm tr}=5000$ are used to train the conditional quantile regressors and propensity models, while the remaining $n^{\rm cal}_{\rm obs}=2000$ are reserved for calibration. In addition, we generate a randomized telemetry dataset containing $n_{\rm rnd}=50$ samples. The benchmarked methods are evaluated over $n_{\rm folds}=100$ independent folds using $n^{\rm te}=1000$ test samples. We set the target miscoverage level to $\alpha=0.1$ and the guardrail parameter to $\epsilon=0.025$.

 We evaluate the benchmarked methods for the potential outcome $Y(1)$ by varying the confounding strength through the parameter $\lambda$ in the handover acceptance probability \eqref{eq:ho-accept}, considering the values $\lambda\in\{0,2,4,6,8\}$. As discussed in Section~\ref{subsubsec:handover-setup}, increasing $\lambda$ strengthens the dependence of the observational handover policy on the unobserved target-cell load $\ell_2$, thereby increasing the degree of hidden confounding.

 \begin{figure}[]
    \centering
    \includegraphics[width=0.99\textwidth]{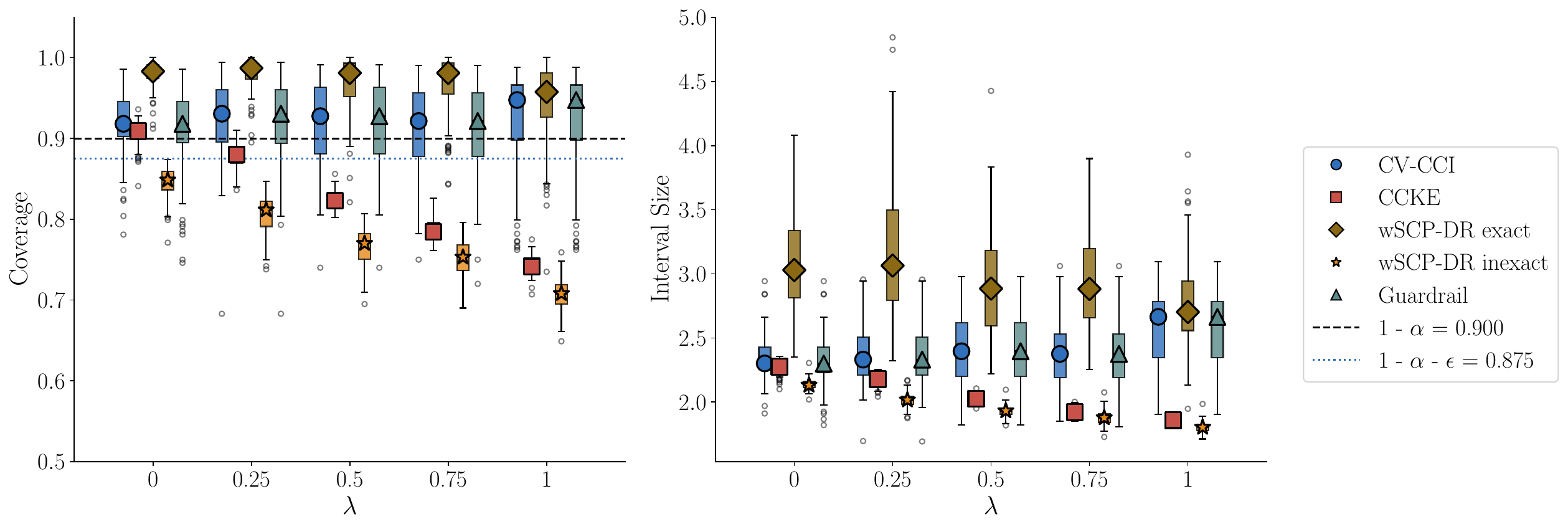}
    \caption{Empirical coverage (left) and prediction-interval efficiency
    (right) for the handover treatment outcome ($Y(1)$) as a function of the
    confounding strength $\lambda$. The dashed line indicates the nominal
    coverage level $1-\alpha=0.90$, while the dotted line indicates the
    guardrail level $1-\alpha-\epsilon=0.875$. Boxplots summarize the
    results over ${n_{\rm folds}} = 100$ folds.}
    \label{fig:ho-t1}
\end{figure}

 For each value of $\lambda$, Figure \ref{fig:ho-t1} reports the empirical coverage and efficiency computed according to \eqref{eq:scheduling-empirical-coverage} and \eqref{eq:scheduling-empirical-efficiency}, with $K=1$ since the target potential outcome is scalar. The results confirm the advantage of the proposed CV-CCI methodology, which consistently achieves the best trade-off between statistical validity and prediction-set efficiency. In particular, when $\lambda=0$ and hidden confounding is absent, all methods except wSCP-DR Inexact achieve an empirical coverage above the nominal level $1-\alpha=0.9$. However, as the confounding strength increases, the coverage of CCKE and wSCP-DR Inexact progressively deteriorates, reaching an empirical coverage level of 0.75 and 0.71 for $\lambda=8$, respectively. Although these methods produce comparatively narrow prediction intervals, they consistently violate the coverage requirement across all values of $\lambda>0$.

By contrast, CV-CCI, Guardrail, and wSCP-DR Exact maintain empirical coverage above the nominal level for all values of $\lambda$. However, while the coverage achieved by CV-CCI and Guardrail remains close to the nominal level, wSCP-DR Exact exhibits noticeable overcoverage, resulting in substantially wider prediction intervals due to its data inefficiency. Overall, CV-CCI achieves the most favorable trade-off between robustness to hidden confounding and prediction-set efficiency, matching the robustness of Guardrail while producing substantially narrower prediction intervals than wSCP-DR Exact.

\begin{table}[]
    \centering
    \caption{Empirical failure rates of CV-CCI and Guardrail for the no-handover potential outcome.}
    \label{tab:ho-fails-y0}
    \begin{tabular}{lcc}
        \toprule
        \textbf{Confounding ($\lambda$)} & \textbf{CV-CCI} & \textbf{Guardrail} \\
        \midrule
        0 & 9\%  & 18\% \\
        2 & 7\%  & 28\% \\
        4 & 12\% & 25\% \\
        6 & 10\% & 18\% \\
        8 & 13\% & 40\% \\
        \bottomrule
    \end{tabular}
\end{table}

Although CV-CCI and Guardrail provide similar empirical mean coverage and efficiency, they do so with markedly different variability. This is illustrated in Table~\ref{tab:ho-fails-y0}, which reports the empirical failure probability of CV-CCI and Guardrail for the potential outcome $Y(0)$. The empirical failure probability is defined as the fraction of the $n_{\rm folds}=100$ folds for which the empirical coverage falls below the guardrail level $1-\alpha-\epsilon=0.875$,  and it is defined as in \eqref{eq:fold-failure-rate}. The results show that CV-CCI consistently achieves a lower empirical failure probability than Guardrail across all confounding levels. In the absence of hidden confounding ($\lambda=0$), Guardrail violates the guardrail level in $18\%$ of the folds, compared with only $9\%$ for CV-CCI. As the confounding strength increases, this gap becomes even more pronounced. Under the strongest confounding level, Guardrail fails to satisfy the prescribed coverage in $40\%$ of the folds, whereas CV-CCI reduces this probability to only $13\%$.

\begin{figure}[t]
    \centering
    \includegraphics[width=\columnwidth]{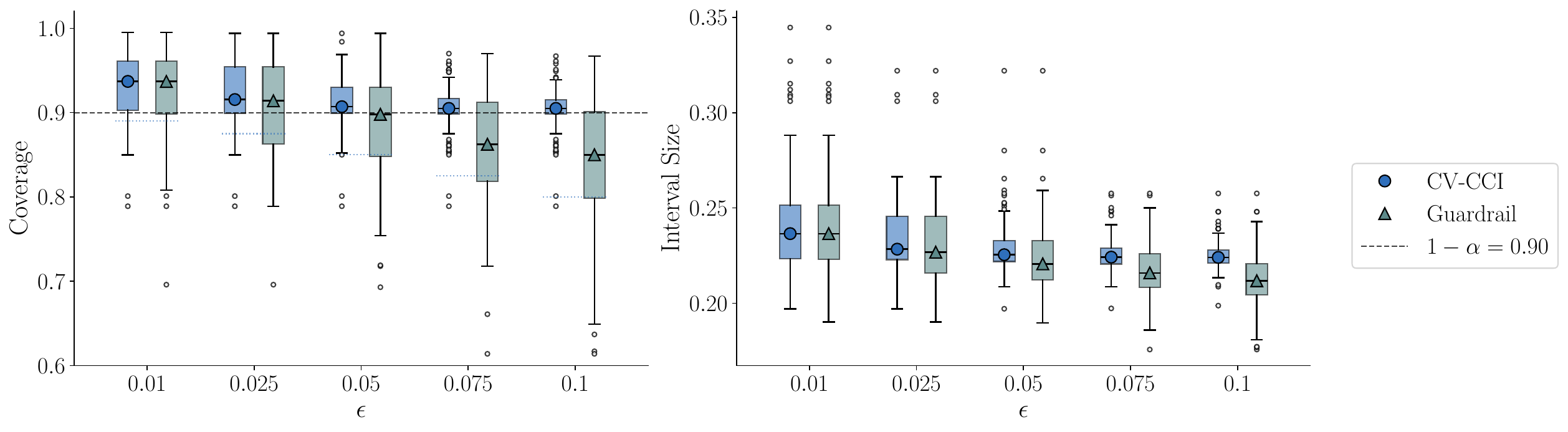}
    \caption{Empirical coverage (left) and efficiency (right) of CV-CCI and Guardrail for the no-handover potential outcome $Y(0)$ as a function of the miscoverage tolerance $\epsilon$.}
    \label{fig:sweep_eps_Y0}
\end{figure}

In Figure \ref{fig:sweep_eps_Y0} we examine whether the reliability advantage of CV-CCI persists across different choices of the tolerance parameter $\epsilon$. Holding the confounding strength fixed at $\lambda=2$, we vary $\epsilon$ over $\{0.01,0.025,0.05,0.075,0.1\}$ and evaluate CV-CCI and Guardrail for the no-handover outcome $Y(0)$. All other experimental parameters remain unchanged. The empirical coverage of Guardrail decreases as $\epsilon$ increases, remaining close to its theoretical guarantee of $1-\alpha-\epsilon$. Its performance also becomes less stable at larger values of $\epsilon$, as indicated by the increasing dispersion across folds. CV-CCI is considerably less sensitive to the choice of $\epsilon$, and its empirical coverage remains close to the nominal level $1-\alpha$. These results highlight the role of the observational component in CV-CCI, which not only reduces variability for a fixed value of $\epsilon$, as observed in Table \ref{tab:ho-fails-y0}, but also enables CV-CCI to achieve empirical coverage closer to the nominal level across different choices of $\epsilon$.

\section{Conclusions}
\label{sec:conclusions}

This paper studied reliable counterfactual inference in wireless networks when the available telemetry does not contain all the information used by the controller. This mismatch gives rise to hidden confounding and can invalidate the coverage guarantees of counterfactual methods that rely on observational telemetry alone. To address this problem, we introduced \emph{Confounding-Valid Counterfactual Conformal Inference} (CV-CCI), a methodology that combines abundant, potentially confounded observational telemetry with limited randomized data to construct counterfactual KPI prediction sets with finite-sample coverage guarantees, up to a user-specified tolerance.

Experiments on MAC-layer scheduling and handover demonstrated the practical effectiveness of CV-CCI. Across different levels of hidden confounding, CV-CCI maintained the prescribed coverage while producing informative prediction sets. By contrast, methods relying primarily on observational telemetry increasingly violated the coverage requirement as the strength of confounding increased, whereas alternative valid methods produced substantially wider prediction sets or exhibited greater variability because of the limited randomized sample size.

Future work will investigate extensions of the proposed framework to sequential and online counterfactual inference under evolving network conditions \cite{gauthier2026backward}, interference-aware settings involving interacting network controllers \cite{del2025pacifista}, and active intervention strategies that adaptively collect the most informative data while minimizing network disruption \cite{pmlr-v235-zrnic24a}.

\bibliographystyle{IEEEtran}

\bibliography{references}
\end{document}